\documentclass[onefignum,onetabnum]{siamart171218}

\usepackage{lipsum}
\usepackage[algo2e]{algorithm2e}
\usepackage{amsfonts}
\usepackage{graphicx}
\usepackage{epstopdf}
\usepackage{algorithmic}
\usepackage{diagrams}
\usepackage{mathtools}
\usepackage{adjustbox}
\usepackage{multirow}
\usepackage{siunitx}

\ifpdf
  \DeclareGraphicsExtensions{.eps,.pdf,.png,.jpg}
\else
  \DeclareGraphicsExtensions{.eps}
\fi

\newif\ifanon

\anonfalse

\ifanon
\newcommand{\authone}{Anonymous Author 1}
\newcommand{\authtwo}{Anonymous Author 2}
\newcommand{\deptone}{Undisclosed Department 1}
\newcommand{\depttwo}{Undisclosed Department 2}
\newcommand{\locone}{Undisclosed Location 1}
\newcommand{\loctwo}{Undisclosed Location 2}
\newcommand{\emailone}{anon1@anon.edu}
\newcommand{\emailtwo}{anon2@anon.edu}

\newcommand{\ackinfo}{Acknowledgements withheld for anonymity.}
\else
\newcommand{\authone}{C.B. Scott}
\newcommand{\authtwo}{Eric Mjolsness}
\newcommand{\deptone}{Department of Computer Science}
\newcommand{\depttwo}{Departments of Computer Science and Mathematics}
\newcommand{\locone}{University of California, Irvine}
\newcommand{\loctwo}{University of California, Irvine}
\newcommand{\emailone}{scottcb@uci.edu}
\newcommand{\emailtwo}{emj@uci.edu}
\newcommand{\ackinfo}{
This work was supported by U.S. National Science Foundation NRT Award number 1633631, 
Human Frontiers Science Program grant HFSP - RGP0023/2018,
U.S. National Institute of Aging grant AG059602,
U.S. National Institutes for Health grant R01HD073179,
USAF/DARPA FA8750-14-C-0011,
and by the Leverhulme Trust and
and the hospitality of the Sainsbury Laboratory Cambridge University.}

\fi

\newcommand{\maketitletext}{Multilevel Artificial Neural Network Training for Spatially Correlated Learning}

\newsiamremark{remark}{Remark}
\newsiamremark{hypothesis}{Hypothesis}
\crefname{hypothesis}{Hypothesis}{Hypotheses}
\newsiamthm{claim}{Claim}

\newsiamthm{thm}{Theorem}

\headers{Multiscale Artificial Neural Networks}{\authone\ and \authtwo}

\title{\maketitletext\thanks{Submitted to the editors 06/18/2018.
\funding{\ackinfo}
}}

\author{\authone\thanks{\deptone, \locone.
  (\email{\emailone}).}
\and \authtwo\thanks{\depttwo, \loctwo. 
  (\email{\emailtwo}).}
}

\usepackage{amsopn}

\makeatletter
\newcommand*{\addFileDependency}[1]{
  \typeout{(#1)}
  \@addtofilelist{#1}
  \IfFileExists{#1}{}{\typeout{No file #1.}}
}
\makeatother

\newcommand*{\myexternaldocument}[1]{%
    \externaldocument{#1}%
    \addFileDependency{#1.tex}%
    \addFileDependency{#1.aux}%
}

\ifpdf
\hypersetup{
  pdftitle=\maketitletext,
  pdfauthor={\authone and \authtwo}
}
\fi

\myexternaldocument{ex_supplement}

\newcommand*{\tplotfactor}{1.0}

\newif\ifseptables
\septablesfalse

\newif\ifnotseptables
\ifseptables
\notseptablesfalse
\else
\notseptablestrue
\fi

\allowdisplaybreaks

\begin{document}

\nocite{*}
\maketitle
%

\begin{abstract}
Multigrid modeling algorithms are a technique used to
accelerate iterative method models running on a hierarchy of similar 
graphlike
structures. We introduce and demonstrate a new method 
for
training neural networks which uses multilevel methods. 
Using an objective function derived from a graph-distance metric, we perform orthogonally-constrained optimization to find optimal prolongation and restriction maps between graphs. We compare and contrast several methods for performing this numerical optimization, and additionally present some new theoretical results on upper bounds of this type of objective function. 
Once calculated, these optimal maps between graphs form the core of Multiscale Artificial Neural Network (MsANN) training, a new procedure we present which simultaneously trains a hierarchy of neural network models of varying spatial resolution. Parameter information is passed between members of this hierarchy according to standard coarsening and refinement schedules from the multiscale modelling literature. In our machine learning experiments, these models are able to learn faster than training at the fine scale alone, achieving a comparable level of error with fewer weight updates (by an order of magnitude).
\end{abstract}

\begin{keywords}
  Multigrid Methods, Neural Networks, Classification, Image Analysis
\end{keywords}

\begin{AMS}
   46N10, 47N10, 65M55, 68T05, 82C32
\end{AMS}

\section{Motivation}
Multigrid methods (or multilevel methods when the underlying graph is not a grid) comprise a modeling framework that seeks to ameliorate a core problem in iterative method models with local update rules: namely, that these models have differing rates of convergence for fine-scale and coarse-scale modes \cite{yserentant1993old}. Because iteration of these models involves making updates of a given characteristic length, they are maximally efficient for propagating modes of approximately this wavelength, but ignore finer modes and are inefficient on coarser modes. Multigrid approaches gain computational benefit by addressing these multiple length-scales of behavior using multiple model resolutions, rather than attempting to address them all at the finest scale (in which the coarse modes converge slowly). These methods make use of ``prolongation'' and ``restriction'' operators to move between models in a hierarchy of scales. At each level, a ``smoothing'' step is performed - usually, for multilevel methods, a smoothing step consists of one pass of some iterative method for improving the model at that scale. 

In this paper, we describe a novel general algorithm for applying this approach to the training of Artificial Neural Networks (ANNs). In particular we demonstrate the efficiency of this new method, which combines ideas from machine learning and multilevel modelling, by training several hierarchies of Autoencoder networks (ANNs which learn a mapping from data to a lower-dimensional latent space) \cite{hinton2006reducing} \cite{bourlard1988auto}. By applying multilevel modeling methods we learn this latent representation with an order of magnitude less cost. We will make our notion of `cost' more precise in the experiments section, Section \ref{sec:netsec}.
\section{Background}

\subsection{Prior Work}
In this section, we discuss prior attempts 
to apply
ideas from multigrid methods to neural network models. Broadly speaking, prior approaches to neural net multigrid can be categorized into two classes: (1) Neural network models which are ``structurally multigrid'', i.e. are typical neural network models which make use of multiple scales of resolution; and (2) Neural network training processes which are hierarchical in some way, or use a coarsening-refinement procedure as part of the training process.

In the first class are approaches \cite{grais2017multi, ke2016neural, serban2017multiresolution}. Reference \cite{ke2016neural} implements a convolutional network in which convolutions make use of a multigrid-like structure similar to a Gaussian pyramid, with the motivation that the network will learn features at multiple scales of resolution.  Reference \cite{grais2017multi} defines a convolution operation, inspired by multigrid methods, that convolves at multiple levels of resolution simultaneously. Reference \cite{serban2017multiresolution} demonstrates a recurrent neural network model which similarly operates in multiple levels of some scale space; but in this paper the scale space is a space of aggregated language models (specifically, the differing scales are different levels of generality in language models - for example, topic models are coarsest, word models are finest, with document models somewhere in between). Common to all three of these approaches is that they make use of a modified neural net structure while leaving the training process unchanged, except that the network accepts multiresolution inputs. 

In contrast, multilevel neural network models \cite{bakshi1993wave, schroder2017parallelizing} in the second category present modified learning procedures which also use methodology similar to multilevel modeling. Reference \cite{bakshi1993wave} introduces a network which learns at coarse scales, and then gradually refines its decision making by increasing the resolution of the input space and learning ``corrections'' at each scale. However, that paper focuses on the capability of a particular family of basis functions for neural networks, and not on the capabilities of the multigrid approach. Reference \cite{schroder2017parallelizing} presents a reframing of the neural network training process as an evolution equation in time, and then applies a method called MGRIT (Multigrid Reduction in Time \cite{falgout2014parallel}) to achieve the same results as parallelizing over many runs of training. 

Our approach is fundamentally different: we use coarsened versions of the network model to make coarse updates to the weight variables of our model, followed by `smoothing steps' in which the fine-scale weights are refined. This approach is more general than any of \cite{grais2017multi, ke2016neural, serban2017multiresolution}, since it can be applied to any feed-forward network and is not tied to a particular network structure. The approach in \cite{schroder2017parallelizing} is to parallelize the training process by reframing it as a continuous-in-time evolution equation, but it still uses the same base model and therefore only learns at one spatial scale.

Our method is both structurally multilevel and learns using a multilevel training procedure. 
Our hierarchical neural network architecture is the first to learn at all spatial scales simultaneously over the course of training, transitioning between neural networks of varying input resolution according to standard multigrid method schedules of coarsening and refinement.
To our knowledge, this represents a fully novel approach to combining the powerful data analysis of neural networks with the model acceleration of multiscale modeling. 

\subsection{Outline}
Section \ref{sec:theory} covers the mathematical theory underlying our method. We first introduce the necessary definitions, which are then used in Subsection \ref{subsub:objective} to define an objective function which evaluates a map between two graphs in terms of how well it preserves the behavior of some local process operating on those graphs (interpreting the smaller of the two graphs as a coarsened version of the larger). In Subsection \ref{subsub:precomputeP} we examine some properties of this objective function, including presenting some projection matrices which are local optima for particular choices of graph structure and process. In Subsections \ref{subsec:msann_defn} and \ref{subsub:msann_training}, we define the \emph{Multiscale Artificial Neural Network (MsANN)}, a hierarchically-structured neural network model which uses these optimized projection matrices to project network parameters between levels of the hierarchy, resulting in more efficient training. In Section \ref{sec:netsec}, we demonstrate this efficiency by training a simple neural network model on a variety of datasets, comparing the cost of our approach to that of training only the finest network in the hierarchy. Finally, we conclude the paper by proving two novel properties of our objective function in Section \ref{sec:theoryshortsec}.

 \section{Theory}
 \label{sec:theory}
In this section, we first define basic terms used throughout the paper, and explain the core theory of our paper: that of optimal prolongation maps between 
computational processes running on graph-based data structures,
and hence between graphs.
In this paper we use a specific example of such a process, single-particle diffusion on graphs, to examine the behavior of these prolongation maps.
Finally, we discuss numerical methods for finding (given two input graphs $G_1$ and $G_2$, and a process) prolongation and restriction maps which minimize the error of using $G_1$ as a surrogate structure for simulating the behavior of that process on $G_2$. We will define more rigorously what we mean by ``process'', ``error'', and ``prolongation'' in Section \ref{subsub:objective}.
\subsection{Definitions}
\label{subsub:definitions}
In order to describe our objective function, we must first introduce some core concepts related to minimal mappings between graphs. 
\begin{itemize}
\item Graph lineage: A graph lineage is a sequence of graphs, indexed by $ l \in \mathbb{N} = 0, 1, 2, 3 \ldots$, satisfying the following:
\begin{itemize}
\item $G_0$ is the graph with one vertex and one self-loop, and;
\item Successive members of the lineage grow roughly exponentially - that is, the growth rate is $O(b^{l+\epsilon})$ for some $b > 1$, $\epsilon \geq 0$, and $l > 1$.
\end{itemize}
We introduce this term to differentiate this definition from that of a graph \emph{family}, which is a sequence of graphs without the growth condition. Most of the graph lineages we examine in this work are structurally similar - for example, the lineage of path graphs of length $2^l$. However, we do not define this similarity in a rigorous sense, and we do not require it in the definition of a lineage.
\item Graph Laplacian: We define the Laplacian matrix of a graph $G$ as $L(G) = A(G) - D(G)$, where $A(G)$ and $D(G)$ are the adjacency matrix and diagonal degree matrix of the graph, respectively. The eigenvalues of this matrix are referred to as the spectrum of $G$. See \cite{belkin2002laplacian, cvetkovic2010introduction} for more details on graph Laplacians and spectral graph theory. Our sign convention for $L$ agrees with the standard continuum Laplacian operator, $\Delta$, of a multivariate function $f$: $\Delta f = \sum_{i=1}^n \frac{\delta^2 f}{\delta x_i^2}$.
\item Kronecker Product and Sum of matrices: Given a $(k \times l)$ matrix $M$, and some other matrix $N$, the Kronecker product is the block matrix
\[
M \otimes N = \begin{bmatrix}
  m_{11} N & \cdots & m_{1l}\mathbf{N} \\
             \vdots & \ddots &           \vdots \\
  m_{k1} N & \cdots & m_{kl} \mathbf{N}
\end{bmatrix}
\]
If $M$ and $N$ are square, their Kronecker Sum is defined, and is given by 
\[M \oplus N = M \otimes I_{N} + I_{M} \otimes N\]
where we write $I_A$ to denote an identity matrix of the same size as $A$.
\item Box Product ($\Box$) of graphs: For $G_1$ with vertex set $U = \{ u_1 , u_2 \ldots \}$ and $G_2$ with vertex set $V = \{ v_1 , v_2 \ldots \}$, $G_1 \Box G_2$ is the graph with vertex set $U \times V$ and an edge between $(u_{i_1}, v_{j_1})$ and $(u_{i_2}, v_{j_2})$ when either of the following is true:
\begin{itemize}
\item $i_1 = i_2$ and $v_{j_1}$ and $v_{j_2}$ are adjacent in $G_2$, or
\item $j_1 = j_2$ and $u_{i_1}$ and $u_{i_2}$ are adjacent in $G_1$.
\end{itemize}
This may be rephrased in terms of the Kronecker Sum $\oplus$ of the two matrices:
\begin{align}
A(G_1 \Box G_2) = A(G_1) \oplus A(G_2) = A(G_1) \otimes I_{|G_2|} + I_{|G_1|} \otimes A(G_2)
\end{align}
\item Cross Product ($\times$) of graphs: For $G_1$ with vertex set $U = \{ u_1 , u_2 \ldots \}$ and $G_2$ with vertex set $V = \{ v_1 , v_2 \ldots \}$, $G_1 \times G_2$ is the graph with vertex set $U \times V$ and an edge between $(u_{i_1}, v_{j_1})$ and $(u_{i_2}, v_{j_2})$ when both of the following are true:
\begin{itemize}
\item $u_{i_1}$ and $u_{i_2}$ are adjacent in $G_1$, and
\item $v_{j_1}$ and $v_{j_2}$ are adjacent in $G_2$.
\end{itemize}
We include the standard pictorial illustration of the difference between these two graph products in Figure \ref{fig:graph_prod}.
\begin{figure}
    \begin{minipage}{.53\linewidth}
    \includegraphics[width=\linewidth]{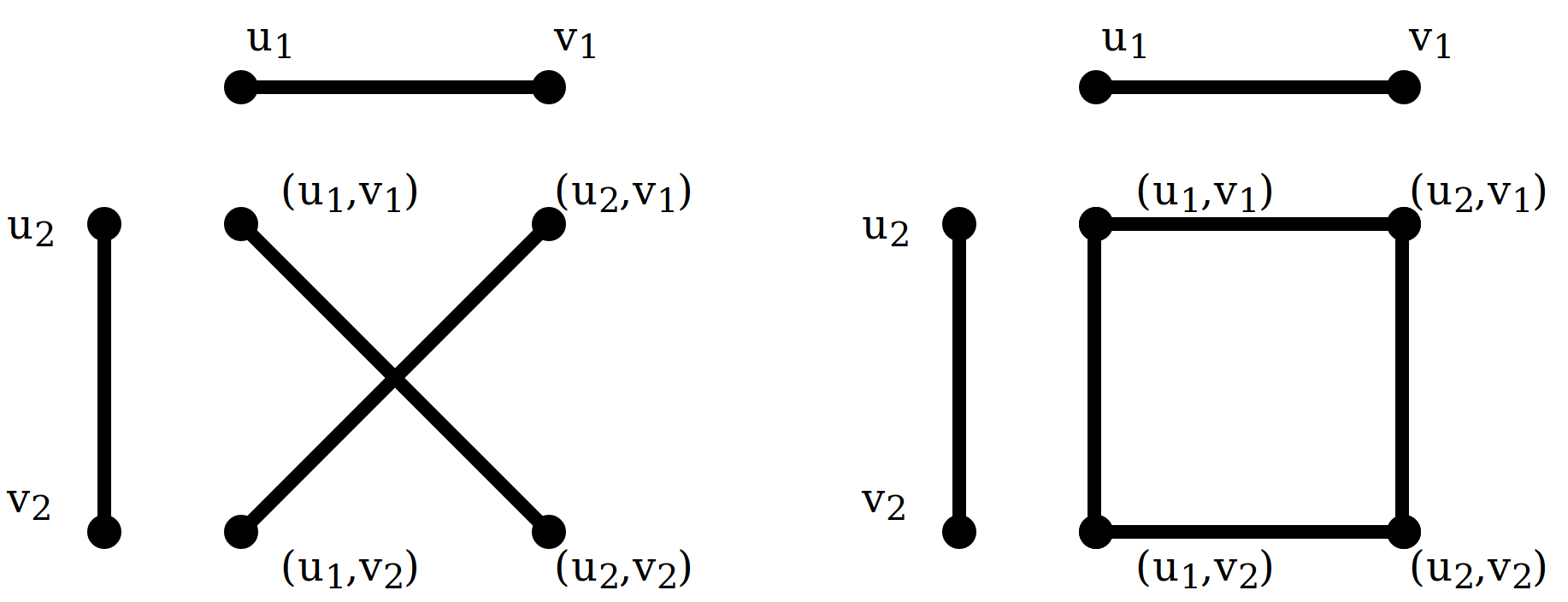}
    \end{minipage}
    \begin{minipage}{.45\linewidth}
        \caption{Two types of graph product: the Cross product ($G_1 \times G_2$, left) and Box product ($G_1 \Box G_2$, right). For two edges $v_1 \sim u_1 \in G_1$ and $v_2 \sim u_2 \in G_2$, we illustrate the resultant edges in the set of vertices $\{ (u_1, v_1), (u_2, v_1), (u_1, v_2), (u_2, v_2) \}$ in the graph product.}
        \label{fig:graph_prod}
    \end{minipage}
\end{figure}
\item Grid Graph: a grid graph (called a lattice graph or Hamming Graph in some texts \cite{brouwer2012distance}) is the distance-regular graph given by the box product of path graphs $P_{a_1}, P_{a_1}, \ldots P_{a_k}$ (yielding a grid with aperiodic boundary conditions) or by a similar list of cycle graphs (yielding a grid with periodic boundary conditions).
\item Prolongation map: A prolongation map between two graphs $G_1$ and $G_2$ of sizes $n_1$ and $n_2$, with $n_2 \geq n_1$, is an $n_2 \times n_1$ matrix of real numbers which is an optimum of the objective function of equation \ref{eqn:objfunction} below (possibly subject to some set of constraints $C(P)$). 
\item Eigenvalue matching: Given two matrices $A_1$ and $A_2$, and lists of their eigenvalues $\{\lambda^{(1)}_1, \lambda^{(1)}_2, \ldots, \lambda^{(1)}_{n_1} \}$ and $\{\lambda^{(2)}_1, \lambda^{(2)}_2, \ldots, \lambda^{(2)}_{n_2} \}$, with $n_2 \geq n_1$, we define the \emph{minimal eigenvalue matching} $m^*(A_1, A_2)$ as the matrix which is the solution of the following constrained optimization problem:
\begin{align}
\label{eqn:matchingconstraints}
  m^*(A_1, A_2) & =  \text{arg} \inf_M \sum_{i = 1}^{n_2} \sum_{j = 1}^{ n_1} M_{i, j} (\lambda^{(1)}_{j} - \lambda^{(2)}_{i})^2 \\
\text{subject to} &\quad  \left( M \in \{ 0, 1\}^{n_2 \times n_1} \right)  \wedge \left( \sum_{i = 1}^{n_2} M_{i,j} = 1 \right) \wedge  \left( \sum_{j = 1}^{ n_1} M_{i,j} \leq 1 \right) \nonumber
\end{align}
In the case of eigenvalues with multiplicity $> 1$, there may not be one unique such matrix, in which case we distinguish matrices with identical cost by the lexicographical ordering of their occupied indices and take $m^*(A_1, A_2)$ as the first of those with minimal cost. This matching problem is well-studied and efficient algorithms for solving it exist; we use a Python language implementation \cite{clapper2008munk} of a 1957 algorithm due to Munkres \cite{munkres1957algorithms}. Additionally, given a way to enumerate the minimal-cost matchings found as solutions to this eigenvalue matching problem, we can perform combinatorial optimization with respect to some other objective function $g$, in order to find optima of $g(P)$ subject to the constraint that $P$ is a minimal matching. 
\end{itemize}

\subsection{Optimal Prolongation Maps Between Graphs}
\subsubsection{Our objective function}
\label{subsub:objective}
Given two graphs $G_1$ and $G_2$, we find the optimal prolongation map between them as follows: We first calculate the graph Laplacians $L_1$ and $L_2$, as well as pairwise vertex Manhattan distance matrices (i.e. the matrix with $T_{i,j}$ the minimal number of graph edges between vertices $i$ and $j$ in the graph), $T_1$ and $T_2$, of each graph. 
Calculating these matrices may not be trivial for arbitrary dense graphs; for example, calculating the pairwise Manhattan distance of a graph with $m$ edges on $n$ vertices can be accomplished in $O(m + n \log n)$  by the Fibonacci heap version of Dijkstra's algorithm \cite{fredman1987fibonacci}. Additionally, in Section $\ref{sec:optresults}$ we discuss an optimization procedure which requires computing the eigenvalues of $L_i$ (which are referred to as the \emph{spectrum} of $G_i$). Computing graph spectra is a well studied problem; we direct the reader to \cite{cohen2018approximating, pan1999complexity}. In practice, all of the graph spectra computed for experiments in this paper took a negligible amount of time ($<$ 1s) on a modern consumer-grade laptop using the scipy.linalg package \cite{jones2001scipy}, which in turn uses LAPACK routines for Schur decomposition of the matrix \cite{lapack1999laug}.
The optimal map is defined as $P$ which minimizes the matrix function
\begin{align}
\label{eqn:objfunction}
& \inf_{P | C(P), \alpha > 0, \beta > 0} && E(P) \hfill & \\
= & \inf_{P | C(P), \alpha > 0, \beta > 0}  && \left[ (1-s) {\left| \left| \frac{1}{\sqrt{\alpha}} P L_1 - \sqrt{\alpha} L_2 P \right|\right|}^2_F \right.  & \text{``Diffusion Term"} \nonumber \\
& && \left.  + s {\left|\left|\frac{1}{\sqrt{\beta}} P T_1 - \sqrt{\beta} T_2 P \right| \right|}^2_F \right] & \text{``Locality Term"\footnotemark} \nonumber 
\end{align} \footnotetext{By this we mean the notion that neighborhoods of $G_1$ should be mapped to neighborhoods of $G_2$ and vice versa.}%
where $|| \cdot ||_F$ is the Frobenius norm, and $C(P)$ is a set of constraints on $P$ (in particular, we require $P^T P = I_{n_1}$, but could also impose other restrictions such as sparsity, regularity, and/or bandedness). The manifold of real-valued orthogonal $n_2 \times n_1$ matrices with $n_1 \leq n_2$ is known as the Stiefel manifold; minimization constrained to this manifold is a well-studied problem \cite{rapcsak2002minimization, turaga2008statistical}. This optimization problem can be thought of as measuring the agreement between processes on each graph, as mapped through $P$. The expression $P X_1 - X_2 P$ compares the end result of 
\begin{enumerate}
\item Advancing process $X_2$ forward in time on $G_2$ and then using $P$ to interpolate vertex states to the smaller graph, to:
\item Interpolating the initial state (the all-ones vector) using $P$ and then advancing process $X_1$ on $G_1$.
\end{enumerate}
Strictly speaking the above interpretation of our objective function does not apply to the Manhattan distance matrix T of a graph, since $T$ is not a valid time evolution operator and thus is not a valid choice for $X$. However, the objective function term containing $T$ may still be interpreted as comparing travel distance in one graph to travel distance in the other. That is, we are implicitly comparing the similarity of two ways of measuring the distance of two nodes $v_k$ and $v_l$ in $G_1$:
\begin{enumerate}
    \item The Manhattan distance, as defined above, and;
    \item $\sum_{i = 1}^{n_2} \sum_{j = 1}^{n_2} p_{ik} d_{G_2}(u_i, u_j) p_{jl}$, a sum of path distances in $G_2$ weighted by how strongly $v_k$ and $v_l$ are connected, through $P$, to the endpoints of those paths, $u_i$ and $u_j$.
\end{enumerate}

Parameters $\alpha$ and $\beta$ are rescaling parameters to compensate for different graph sizes; in other words, $P$ must only ensure that processes 1 and 2 above agree up to some multiplicative constant. In operator theory terminology, the Laplacian is a time evolution operator for the single particle diffusion equation: $L_i = A(G_i) - \text{diag}(1 \cdot A(G_i))$. This operator evolves the probability distribution of states of a single-particle diffusion process on a graph $G_i$ (but other processes could be used - for example, a chemical reaction network or multiple-particle diffusion). The process $L$ defines a probability-conserving Master Equation of nonequilibrium statistical mechanics $dp/dt = L \cdot p$ which has formal solution $p(t) = \exp{(t L)} \cdot p(0)$. Pre-multiplication by the prolongation matrix $P$ is clearly a linear operator i.e. linear transformation from $\mathbb{R}^{n_1}$ to $\mathbb{R}^{n_2}$. Thus, we are requiring $P$ which minimizes the degree to which the operator diagram

\begin{diagram}
L_1         &\rTo^{\Delta t}   & {L_1}'\\
\dTo_{P}  &           &\dTo_{P} &&&& (\text{Diagram }1) \\
L_2        &\rTo^{\Delta t}   & {L_2}'\\
\end{diagram}
fails to commute. $\Delta t$ of course refers to advancement in time. See \cite{johnson2015model}, Figure 1, for a more complete version of this commutative diagram for model reduction.

We thus include in our objective function terms with 1) graph diffusion and 2) graph locality as the underlying process matrices ($T$, the Manhattan distance matrix, cannot be considered a time evolution operator because it is not probability-preserving). Parameter $s$ adjusts the relative strength of these terms to each other; so we may find ``fully diffuse'' $P$ when $s = 0$ and ``fully local'' $P$ when $s = 1$. Figure \ref{fig:localityfig} illustrates this tradeoff for an example prolongation problem on a pair of grid graphs, including the transition from a global optimum of the diffusion term to a global optimum of the locality term. In each case, we only require $P$ to map these processes into one another up to a multiplicative constant: $\alpha$ for the diffusion term and $\beta$ for the locality term. Exhaustive grid search over $\alpha$ and $\beta$ for a variety of prolongations between (a) path graphs and (b) 2D grid graphs of varying sizes has suggested that for prolongation problems where the $G_i$ are both paths or both grids, the best values (up to the resolution of our search, $10^{-6}$) for these parameters are $\alpha=1.0$ and $\beta=n_1/n_2$. 
However, we do not expect this scaling law to hold for general graphs.

\begin{figure}
\label{fig:localityfig}
\begin{center}
\includegraphics[width=\linewidth]{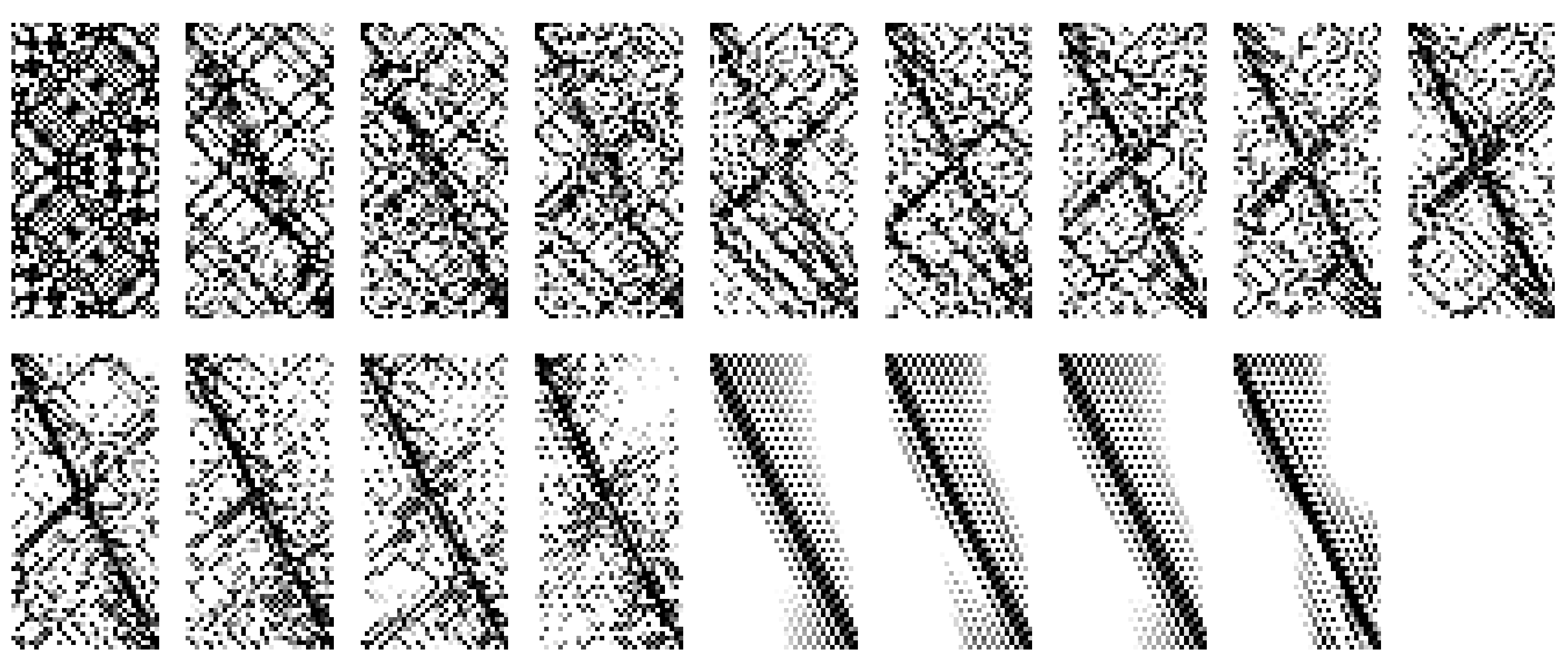}
\end{center}
\caption{Several solutions of our objective function found by PyManOpt as $s$, the relative weight of the two terms of our objective function, is tuned from 0 (fully diffuse, top left) to 1 (fully local, bottom right). Within each subplot, grayscale indicates the magnitude of matrix entries. Note that the $P$ matrices found with $s=0$ do not appear to be structured in a way which respects the locality of the original graphs, whereas the matrices with $s=1$ do.}
\end{figure}

\subsection{Numerical Optimization of P Matrices}
\label{sec:optresults}
\subsubsection{Minimization method}
We tried various publicly available optimization codes to find optima of our objective function. Unless otherwise noted, all $P$ matrices found via optimization were found using PyManOpt, a Python language package for manifold-constrained optimization. In our experience, this package outperformed other numerical methods codes such as constrained Nelder-Mead (as implemented in Mathematica or SciPy), gradient descent with projection back to the constraint manifold, or the orthogonally-constrained optimization code of \cite{wen2013feasible}. More details on our comparison of these software packages may be found in the section ``Comparison of Numerical Methods'' of the Supplementary Material accompanying this paper. 

\subsubsection{Initialization}
We initialize our minimization with an upper-bound solution given by the Munkres minimum-cost matching algorithm; the initial $P$ is $m^*(L_1, L_2)$ as defined in equation \ref{eqn:matchingconstraints}, i.e. the binary matrix where an entry $P_{(i,j)}$ is 1 if the pair $(i,j)$ is one of the minimal-cost pairs selected by the minimum-cost assignment algorithm, and 0 otherwise. While this solution is, strictly speaking, minimizing the error associated with mapping the spectrum of one graph into the spectrum of the other (rather than actually mapping a process running on one graph into a process on the other) we found it to be a reasonable initialization, outperforming both random restarts and initialization with the appropriately sized block matrix $\left( \begin{array}{c}
I \\
0 
\end{array}\right)$. As detailed further in Section \ref{sec:theoryshortsec}, the $P$ found as a solution to this matching problem provides an upper bound for the full orthogonality-constrained optimization problem.

\subsubsection{Precomputing $P$ matrices}
\label{subsub:precomputeP}
For some structured graph lineages it may be possible to derive formulaic expressions for optimal $P$ and $\alpha$, as a function of the lineage index. For example, during our experiments we discovered species of $P$ which are local minima of prolongation between path graphs, cycle graphs, and grid graphs. A set of these outputs is shown in Figure \ref{fig:localityfig}. They feature various diagonal patterns as naturally idealized in Figure \ref{fig:pspeciesfig}. These idealized versions of these patterns all are also empirical local minima of our optimization procedure, for $s=0$ or $s=1$, as indicated. Each column of Figure \ref{fig:pspeciesfig} provides a regular family of $P$ structures for use in our subsequent experiments in Section \ref{sec:netsec}. We have additionally derived closed-form expressions for global minima of the diffusion term of our objective function for some graph families (cycle graphs and grid graphs with periodic boundary conditions). Proof of the optimality of these solutions may be found in the supplementary materials which accompany this paper. However, in practice these global minima are nonlocal (in the sense that they are not close to optimizing the locality term) and thus may not preserve learned spatial rules between weights in levels of our hierarchy. 

Examples of these formulaic $P$ matrices can be seen in Figure \ref{fig:pspeciesfig}. Each column of that figure shows increasing sizes of $P$ generated by closed-form solutions which were initially found by solving smaller prolongation problems (for various graph pairs and choices of $s$) and generalizing the solution to higher $n$. Many of these examples are similar to what a human being would design as interpolation matrices between cycles and periodic grids. However, (a) they are valid local optima found by our optimization code and (b) our approach generalizes to processes running on more complicated or non-regular graphs, for which there may not be an obvious \textit{a priori} choice of prolongation operator. 

We highlight the best of these multiple species of closed-form solution, for both cycle graphs and grid graphs. The interpolation matrix-like $P$ seen in the third column of the ``Cycle Graphs" section, or the sixth column of the ``Grid Graphs'' section of Figure \ref{fig:pspeciesfig}, were the local optima with lowest objective function value (with $s = 1$, i.e. they are fully local). As the best optima found by our method(s), these matrices were our choice for line graph and grid graph prolongation operators in our neural network experiments, detailed in Section \ref{sec:netsec}. 
We reiterate that in those experiments we do not find the $P$ matrices via any optimization method - since the neural networks in question have layer sizes of order $10^3$, finding the prolongation matrices from scratch may be computationally difficult. Instead, we use the solutions found on smaller problems as a recipe for generating prolongation matrices of the proper size.
Furthermore, given two graph lineages $G_1^{(1)}, G_1^{(2)}, G_1^{(3)} \ldots$ and $G_2^{(1)}, G_2^{(2)}, G_2^{(3)} \ldots$, and sequences of optimal matrices $P_1^{(1)}, P_1^{(2)}, P_1^{(3)} \ldots$ and $P_2^{(1)}, P_2^{(2)}, P_2^{(3)} \ldots$ mapping between successive members of each, we can construct $P$ which are related to the optima for prolonging between members of a new graph lineage which is comprised of the levelwise graph box product of the two sequences. We show in (Section \ref{subsec:deompose}, Corollary \ref{thm:decompcorollary}) conditions under which the value of the objective function at $P_{\text{box}}^{(i)} = P_1^{(i)} \otimes P_2^{(i)}$ is an upper bound of the optimal value for prolongations between members of the lineage $G_1^{(1)} \Box G_2^{(1)}, G_1^{(2)} \Box G_2^{(2)}, G_1^{(3)} \Box G_2^{(3)}, \ldots$ . We leave open the question of whether such formulaic $P$ exist for other families of structured graphs (complete graphs, $k$-partite graphs, etc.). Even in cases where formulaic $P$ are not known, the computational cost of numerically optimizing over $P$ may be amortized, in the sense that once a $P$-map is calculated, it may be used in many different hierarchical neural networks or indeed many different multiscale models. 

\begin{figure}
\label{fig:pspeciesfig}
\begin{center}
\hfil
\includegraphics[width=.4\linewidth]{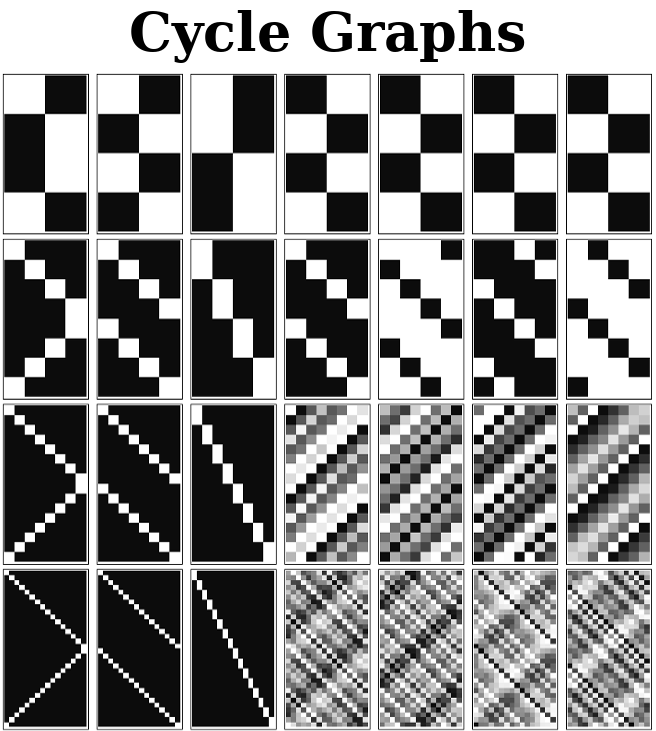}
\hfil
\includegraphics[width=.4\linewidth]{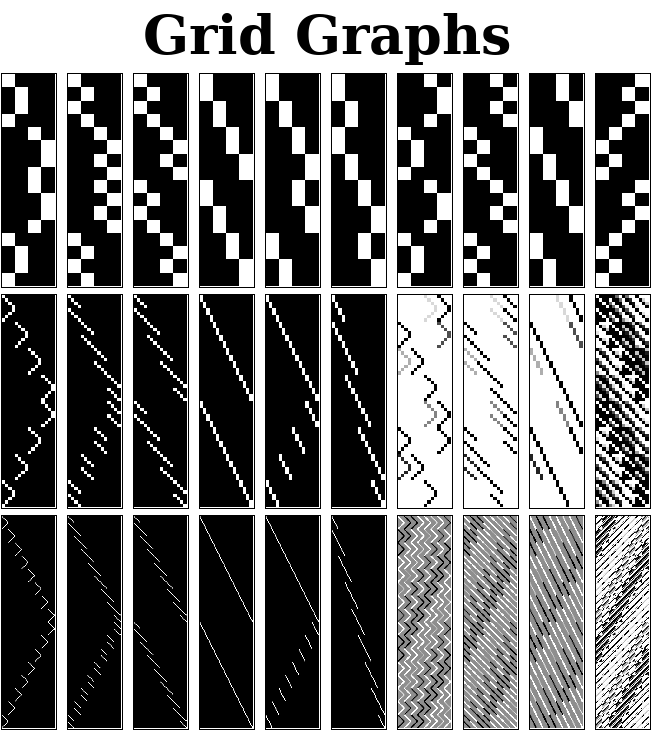}
\hfil
\end{center}
\caption{Examples of $P$ matrices for cycle graph (left) and grid graph (right) prolongation problems of various sizes, which can be generated by closed-form representations dependent on problem size. Within each of the top and bottom plots, columns represent a series of matrices each generated by a particular numerical recipe, with rows representing increasing sizes of prolongation problem. Each matrix plot is a plot of the absolute value of matrix cell values. These closed-form representations were initially found as local minima of our objective function on small problems and then generalized to closed-form representations. For the ``Cycle Graphs'' section, the prolongation problems were between cycle graphs of sizes $n_1 = 2,4,8,16$ and $n_2 = 2*n_1$. Columns 1-3 were solutions found with $s=1$ (fully local), and the rest were found with $s= 0$ (fully diffuse). For the ``Grid Graphs'' section, the prolongation problems were between grids of size $(n_1, n_1)$ to grids of size $(2 n_1, 2 n_1)$ for $n_1$ in $4,8,16$. Columns 1-6 are fully local and columns 7-10 are fully diffuse, respectively. As in Figure \ref{fig:localityfig}, grayscale values indicate the magnitude of each matrix entry.}
\end{figure}

\subsection{Multiscale Artificial Neural Network Algorithm}
\label{subsec:msann_defn}
In this section we describe the Multiscale Artificial Neural Network (MsANN) training procedure, both in prose and in pseudocode (Algorithm \ref{alg:mgannalg}). Let $\mathcal{M}_0 \ldots \mathcal{M}_{L}$ be a sequence of neural network models with identical ``aspect ratios'' (meaning the sizes of each layer relative to other layers in the same model) but differing input resolution, so that $\mathcal{M}_0$ operates at the finest scale and $\mathcal{M}_L$ at the coarsest. For each model $\mathcal{M}_l$, let $\theta^{(l)}_0, \theta^{(l)}_1, \ldots \theta^{(l)}_{n_\text{vars}-1}$ be a list of the $n_\text{vars}$ network parameters (each in matrix or vector form) in some canonical order which is maintained across all scales. Let the symbol $\mathcal{P}^{(l)}_j$ represent either: 
\begin{itemize}
\item If the network parameters $\theta^{(i)}_j$ at levels $i = 0 \ldots L$ are weight matrices between layers $m_1$ and $m_2$ of each hierarchy, then $\mathcal{P}^{(l)}_j$ represents a pair of matrices $\left( P^{(l)}_{\text{input}_j}, P^{(l)}_{\text{output}_j} \right)$, such that:
\begin{itemize}
\item $P^{(l)}_{\text{input}_j}$ prolongs or restricts between possible values of nodes in layer $m_1$ of model $\mathcal{M}_l$, and values of nodes in layer $m_1$ of model $\mathcal{M}_{l+1}$.
\item $P^{(l)}_{\text{output}_j}$ does the same for possible values of nodes in layer $m_2$ of each model.
\end{itemize}
\item If the network parameters $\theta^{(i)}_j$ at levels $i = 0 \ldots L$ are bias vectors which are added to layer $m$ of each hierarchy, then $\mathcal{P}^{(l)}_j$ represents a single $P^{(l)}_j$ which prolongs or restricts between possible values of nodes in layer $m$ of model $\mathcal{M}_l$, and values of nodes in layer $m$ of model $\mathcal{M}_{l+1}$.
\end{itemize}

As a concrete example, for a hierarchy of single-layer networks $\mathcal{M}_0, \mathcal{M}_1, \mathcal{M}_2$,  each with one weight matrix $W^{(l)}$ and one bias vector $b^{(l)}$, we could have $\theta^{(l)}_0 = W^{(l)}, \theta^{(l)}_1 = b^{(l)}$ for each $\mathcal{M}_l$. $\mathcal{P}^{(0)}_0$ would represent a pair of matrices which map between the space of possible values of $ W^{(0)}$ and the space of possible values of $ W^{(1)}$ in a manner detailed in the next section. On the other hand, $\mathcal{P}^{(0)}_1$ would represent a single matrix which maps between $b^{(0)}$ and $b^{(1)}$.  Similarly, $\mathcal{P}^{(1)}_0$ would map between $ W^{(1)}$ and $ W^{(2)}$, and $\mathcal{P}^{(1)}_1$ between $b^{(1)}$ and $b^{(2)}$. In Section \ref{subsub:msann_training}, we describe a general procedure for training such a hierarchy according to standard multilevel modeling schedules of refinement and coarsening, with the result that the finest network, informed by the weights of all coarser networks, requires fewer training examples.
\subsubsection{Weight Prolongation and Restriction Operators}
In this section we introduce the prolongation and restriction operators for neural network weight and bias optimization variables in matrix or vector form respectively. 

For a 2D matrix of weights $W$, define
\begin{gather}
\label{eqn:weightpro}
\begin{aligned}
\text{Pro}_{\mathcal{P}} \circ W &\equiv
\text{Pro}_{(P_{\text{input}}, P_{\text{output}})} \circ W &\equiv P_{\text{input}} W P_{\text{output}}^T \\
\text{Res}_{\mathcal{P}} \circ W &\equiv 
\text{Res}_{(P_{\text{input}}, P_{\text{output}})} \circ W &\equiv P_{\text{input}}^T W P_{\text{output}}
\end{aligned}
\end{gather}
where $P_{\text{input}}$ and $P_{\text{output}}$ are each prolongation maps between graphs which respect the structure of the spaces of inputs and outputs of $W$, i.e. whose structure is similar to the structure of correlations in that space.
Further research is necessary to make this notion more precise. In our experiments on autoencoder networks in Section \ref{sec:netsec}, we use example problems with an obvious choice of graph to use. In these 1D and 2D machine vision tasks, where we expect each pixel to be highly correlated with the activity of its immediate neighbors in the grid, 1D and 2D grids are clear choices of graphs for our prolongtion matrix calculation. Other choices may lead to similar results; for instance, we speculate that since neural network weight matrices may be interpreted as the weights of a multipartite graph of connected neurons in the network, these graphs could be an alternate choice of structure to prolong/restrict between.  We leave for future work the development of automatic methods for determining these structures. 

Note that the Pro and Res linear operators satisfy
$ \text{Res}_{\mathcal{P}} \circ \text{Pro}_{\mathcal{P}} = I$,
the identity operator, so 
$ \text{Pro}_{\mathcal{P}} \circ \text{Res}_{\mathcal{P}}$
is a projection operator.
%
%

For a 1D matrix of biases $b$, define
\begin{equation}
\label{eqn:biaspro}
\begin{split}
\text{Pro}_{\mathcal{P}} \circ b &= P \cdot b \\
\text{Res}_{\mathcal{P}} \circ b &= P^T \cdot b
\end{split}
\end{equation}
where, as before, we require that $P$ be a prolongation matrix between graphs which are appropriate for the dynamics of the network layer where $b$ is applied. Again $ \text{Res}_{\mathcal{P}} \circ \text{Pro}_{\mathcal{P}} = I$.

Given such a hierarchy of models $\mathcal{M}_0 \ldots \mathcal{M}_L $, and appropriate \text{Pro} and \text{Res} operators as defined above, we define a \emph{Multiscale Artificial Neural Network (MsANN)} to be a neural network model with the same layer and parameter dimensions as the largest model in the hierarchy, where each layer parameter $\Theta_j$ is given by a sum of prolonged weight matrices from level $j$ of each of the models defined above: 
\begin{align}
\label{eqn:hierarch_var_eqn}
\Theta_j = \theta^{(0)}_j &+ \text{Pro}_{1 \rightarrow 0} \circ \theta^{(1)}_j + \text{Pro}_{2 \rightarrow 0} \circ \theta^{(2)}_j \ldots \text{Pro}_{L \rightarrow 0} \circ \theta^{(L)}_j \\
\intertext{Here we are using $\text{Pro}_{k \rightarrow 0}$ as a shorthand to indicate composed prolongation from model $k$ to model $0$, so if $\theta^{(i)}_j$ are weight variables we have (by Equation \ref{eqn:weightpro}) }
\Theta_j = \theta^{(0)}_j &+ P^{(0)}_{\text{input}_j} \theta^{(1)}_j {\left( P^{(0)}_{\text{output}_j} \right)}^T \\ 
 &+ P^{(0)}_{\text{input}_j} P^{(1)}_{\text{input}_j} \theta^{(2)}_j {\left( P^{(1)}_{\text{output}_j} \right)}^T {\left( P^{(0)}_{\text{output}_j} \right)}^T  \nonumber \\
 & + \quad \ldots \quad + \left( P^{(0)}_{\text{input}_j} \ldots P^{(L-1)}_{\text{input}_j} \theta^{(L)}_j  {\left( P^{(L-1)}_{\text{output}_j} \right)}^T \ldots {\left( P^{(0)}_{\text{output}_j} \right)}^T \right) \nonumber 
 \intertext{and if  $\theta^{(i)}_j$ are bias variables we have (by Equation \ref{eqn:biaspro})}
 \Theta_j = \theta^{(0)}_j & + P^{(0)}_{\text{bias}_j} \theta^{(1)}_j + P^{(0)}_{\text{bias}_j} P^{(1)}_{\text{bias}_j} \theta^{(2)}_j +  \ldots  + \left( P^{(0)}_{\text{bias}_j} P^{(1)}_{\text{bias}_j} \ldots P^{(L-1)}_{\text{bias}_j} \theta^{(L)}_j \right)
\end{align}
We note that matrix products such as $P^{(0)}_{\text{input}_j} \ldots P^{(k)}_{\text{input}_j}$ need only be computed once, during model construction.

\subsubsection{Multiscale Artificial Neural Network Training}
\label{subsub:msann_training}
The Multiscale Artificial Neural Network algorithm is defined in terms of a recursive `cycle' that is analogous to one epoch of default neural network training. Starting with $\mathcal{M}_0$ (i.e. the finest model in the hierarchy), we call the routine $\text{MsANNCycle}(0)$, which is defined recursively. At any level $l$, MsANNCycle trains the network at level $l$ for $k$ batches of training examples, recurses by calling $\text{MsANNCycle}(l + 1)$, and then returns to train for $k$ further batches at level $l$. The number of calls to $\text{MsANNCycle}(l + 1)$ inside each call to $\text{MsANNCycle}(l)$ is given by a parameter $\gamma$.

This is followed by additional training at the refined scale; this process is normally \cite{vanvek1996algebraic} referred to by the multigrid methods community as `restriction' and `prolongation' followed by `smoothing'. The multigrid methods community additionally has special names for this type of recursive refining procedure with $\gamma = 1$ (``V-Cycles'') and $\gamma = 2$ (``W-Cycles''). See Figure \ref{fig:gammafig} for an illustration of these contraction and refinement schedules. In our numerical experiments below, we examine the effect of this parameter on multigrid network training.

Neural network training with gradient descent requires computing the gradient of the error $E$ between the network output and target with regard to the network parameters. This gradient is computed by taking a vector of error for the nodes in the output layer, and \emph{backpropagating} that error backward through the network layer by layer to compute the individual weight matrix and bias vector gradients. An individual network weight or bias term $w$ is then adjusted using gradient descent, i.e. the new value $w'$ is given by $w' = w - \eta \frac{d E}{d w}$, where $\eta$ is a learning rate or step size. Several techniques can be used to dynamically change learning rate during model training - we refer the reader to \cite{bishop2006pattern} for a description of these techniques and backpropagation in general.

Our construction of the MsANN model above did not make use of the $\text{Res}$ (restriction) operator - we show here how this operator is used to compute the gradient of the coarsened variables in the hierarchy. This can be thought of as continuing the process of backpropagation through the $\text{Pro}$ operator. For these calculations we assume $\Theta_j$ is a weight matrix, and derive the gradient for a particular $\theta^{(k)}_j$. For notational simplicity we rename these matrices $W$ and $V$, respectively. We also collapse the matrix products
\begin{align}
P^\text{(input)} &= P^{(0)}_{\text{input}_j} P^{(1)}_{\text{input}_j} \ldots P^{(k)}_{\text{input}_j} \\
{\left( P^\text{(output)} \right)}^T &= {\left( P^{(L-1)}_{\text{output}_j} \right)}^T {\left( P^{(L-2)}_{\text{output}_j} \right)}^T \ldots {\left( P^{(0)}_{\text{output}_j} \right)}^T
\end{align}
Let $\frac{d E}{d W}$ be a matrix where ${\left(\frac{d E}{d W} \right)}_{mn} = \frac{d E}{d w_{mn}}$, calculated via backpropagation as described above. Then, for some $m,n$:
\begin{align}
    \frac{d w_{mn}}{d v_{kl}} &= \frac{d }{d v_{kl}} {\left( \ldots + \text{Pro}_{} \circ V +\ldots \right)}_{mn} \\ 
    &= \frac{d }{d v_{kl}} {\left( \ldots + \text{Pro}_{k \rightarrow 0} \circ V +\ldots \right)}_{mn} = \frac{d }{d v_{kl}} {\left( \text{Pro}_{k \rightarrow 0} \circ V \right)}_{mn} \nonumber \\
    &= \frac{d }{d v_{kl}} {\left( P^\text{(input)} V {\left( P^\text{(output)} \right)}^T \right)}_{mn} 
    = \frac{d }{d v_{kl}} \left( \sum_{a,b} p^\text{(input)}_{ma} v_{ab} p^\text{(output)}_{nb}  \right) \nonumber \\
    &= \left( p^\text{(input)}_{mk} p^\text{(output)}_{nl}  \right) \nonumber 
\end{align}
Then, 
\begin{align}
    \frac{d E}{d v_{kl}} &= \sum_{m,n} \frac{d E}{d w_{mn}} \frac{d w_{mn}}{d v_{kl}} \\
    &= \sum_{m,n} \frac{d E}{d w_{mn}} p^\text{(input)}_{mk} p^\text{(output)}_{nl} \nonumber \\
    &= {\left( {\left( P^\text{(input)} \right)}^T \frac{d E}{d W} P^\text{(output)} \right)_{kl}} \nonumber \\
    \intertext{and so}
    \frac{d E}{d V} &= {\left( P^\text{(input)} \right)}^T \frac{d E}{d W} P^\text{(output)}\nonumber \\
    \intertext{and therefore finally}
    \label{eqn:msann_res_equation} 
    \frac{d E}{d V} &= \text{Res}_{0 \rightarrow k} \circ \frac{d E}{d W} 
\end{align}
where $\text{Res}$ is as in \ref{eqn:weightpro}.

\SetKwFunction{mgann}{MsANNCycle}%
\SetKwProg{Fn}{Procedure}{\string:}{}
\begin{algorithm}[H]
 \Fn{\mgann{$l$}}{
 Train model $\mathcal{M}_l$ for $k$ batches, where each consists of:
 \begin{enumerate}
    \item Feed examples through the network in feed-forward mode;
    \item Compute error $E$ between network output and target;
    \item Use the classical backpropagation algorithm to compute the gradient of top-level parameter $\Theta_j$ w.r.t. this error;
    \item Use the appropriate $\text{Res}$ operations to compute the gradient of $E$ w.r.t. the parameters in $\mathcal{M}_l$, as described in Equation \ref{eqn:msann_res_equation}.
 \end{enumerate}
 \If{\emph{max\_depth} has not been reached}{
  \For{$1 \leq i \leq \gamma$}{
    \mgann{$l+1$};\\
    Train model $\mathcal{M}_l$ for $k$ batches, as above
  }
 }\;}
 \label{alg:mgannalg}
 \caption{One `cycle' of the MsANN procedure.}
\end{algorithm}

We also note here that our code implementation of this procedure does not make explicit use of the $\text{Res}$ operator; instead, we use the automatic differentiation capability of Tensorflow \cite{abadi2016tensorflow} to compute this restricted gradient. This is necessary because data is supplied to the model, and error is calculated, at the finest scale only. Hence we calculate the gradient at this scale and restrict it to the coarser layers of the model. It may be possible to feed coarsened data through only the coarser layers of the model, eliminating the need for computing the gradient at the finest scale, but we do not explore this method in this paper.

\begin{figure}
\label{fig:gammafig}
\begin{center}
\includegraphics[width=\linewidth]{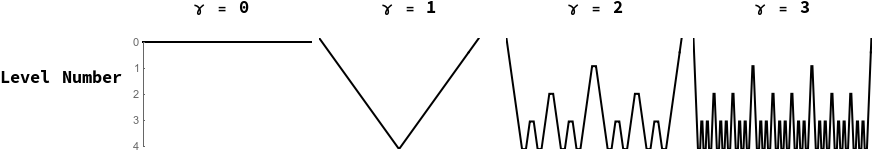}
\end{center}
\caption{Visits to models in a hierarchy of neural networks realized by several values of the recursion frequency parameter $\gamma$. The $\gamma = 1$ case and the $\gamma = 2$ case are referred to as ``V-cycles'' and ``W-cycles'', respectively. Each time the multilevel training procedure visits a level, it performs some number, $k$, of smoothing steps (i.e. gradient descent at that resolution) at that model.}
\end{figure}

\section{Machine Learning Experiments}
\label{sec:netsec}
\subsection{Preliminaries}
\label{subsec:mle_prelim}
We present four experiments
using this Multiscale Neural Network method. All of the experiments below demonstrate that our multigrid method outperforms default training (i.e. training only the finest-scale network), in terms of the number of training examples 
(summed over all scales)
needed to reach a particular mean-squared error (MSE) value. We perform two experiments with synthetic machine vision tasks, as well as two experiments with benchmark image datasets for machine learning. While all of the examples presented here are autoencoder networks (networks whose training task is to reproduce their input at the output layer, while passing through a bottleneck layer or layers), we do not mean to imply that MsANN techniques are constrained to autoencoder networks. All network training uses the standard backpropagation algorithm to compute training gradients, and this is the expected application domain of our method. 
Autoencoding image data is a good choice of machine learning task for our experiments for two main reasons. First, autoencoders are symmetric and learn to reproduce their input at their output. Other ML models (for instance, neural networks for classification) have output whose nodes are not spatially correlated, and it is not yet clear if our approach will generalize to this type of model. Secondly, since the single and double-object machine vision tasks operate on synthetic data, we can easily generate an arbitrary number of samples from the data distribution, which was useful in the early development of this procedure. Our initial successes on this synthetic data led us to try the same task with a standard benchmark real-world dataset.
For each experiment, we use the following measure of computational cost to compare relative performance. Let $\left| \mathcal{M} \right|$ be the number of trainable parameters in model $\mathcal{M}$. We compute the cost of a training step of the weights in model $\mathcal{M}_k$ using a batch of size $b$ as $\frac{|\mathcal{M}_k|}{|\mathcal{M}_0|} b$. The total cost $C(t)$ of training at step $t$ is the sum of this cost over all training steps thus far at all scales. This cost is motivated by the fact that the number of multiply operations for backpropagation is $O(n m)$ in the total number of network parameters $m$ and training examples $n$, so we are adding up the relative cost of using a batch of size $b$ to adjust the weights in model $\mathcal{M}_k$, as compared to the cost of using that same batch to adjust the weights in $\mathcal{M}_0$.

\subsection{Simple Machine Vision Task}
\label{subsec:auto}
As an initial experiment in the capabilities of hierarchical neural networks, we first try two simple examples: finding lower-dimensional representation of two artificial datasets. In both cases, we generate synthetic data by uniformly sampling from 
\begin{enumerate}
\item the set of binary-valued vectors with one ``object'' comprising a contiguous set of pixels one-eighth as long as the entire vector set to 1, and the rest zero; and
\item the set of vectors with two such non-overlapping objects.
\end{enumerate}
In each case, the number of possible unique data vectors is quite low: for inputs of size 1024, we have 1024 - 128 = 896 such vectors. Thus, for both of the synthetic datasets we add binary noise to each vector, where each ``pixel" of the input has an independent chance of firing spuriously with $p=0.05$. This noise in included only in the input vector, making these networks \emph{Denoising Autoencoders}: models whose task is to remove noise from an input image.  
\subsubsection{Single-Object Autoencoder}

We first test the performance of this procedure on a simple machine vision task. The neural networks in our hierarchy of models each have layer size specification (in number of units) $[2^n, 2^{n-2}, 2^{n-3} ,2^{n-2}, 2^n]$ for $n$ in $\{10, \ldots 6\}$, with a bias term at each layer and sigmoid logistic activation.  We present the network with binary vectors which are 0 everywhere except for a contiguous segment of indices of length $2^{n-3}$ which are set to 1, with added binary noise as described above. The objective function to minimize is the mean-squared error (MSE) between the input and output layers. Each model in the hierarchy is trained using 
RMSPropOptimizer in Tensorflow, with learning rate $\alpha = 0.0005$.

The results of this experiment are plotted in Figure \ref{fig:autofig} and summarized in Table \ref{tbl:oneobj}. We perform multiple runs of the entire training procedure with differing values of $k$ (the number of smoothing steps), $\gamma$ (the multigrid cycle parameter), and $L$ (depth of hierarchy).  Notably, nearly all multigrid schedules demonstrate performance gains over the default network (i.e. the network which trains only at the $l=0$ scale), with more improvement for higher values of $k$, $L$, and $\gamma$. The hierarchy which learned most rapidly was the deepest model $(L=6)$ with $k = 4$ and $\gamma = 3$. Those multigrid models which did not improve over the default network were only slightly more computationally expensive per unit of accuracy than their default counterparts, and the multigrid models which did improve, improved significantly.

\begin{figure}
\label{fig:autofig}
\centering
\begin{minipage}{.47\linewidth}
	\includegraphics[width=\tplotfactor\linewidth]{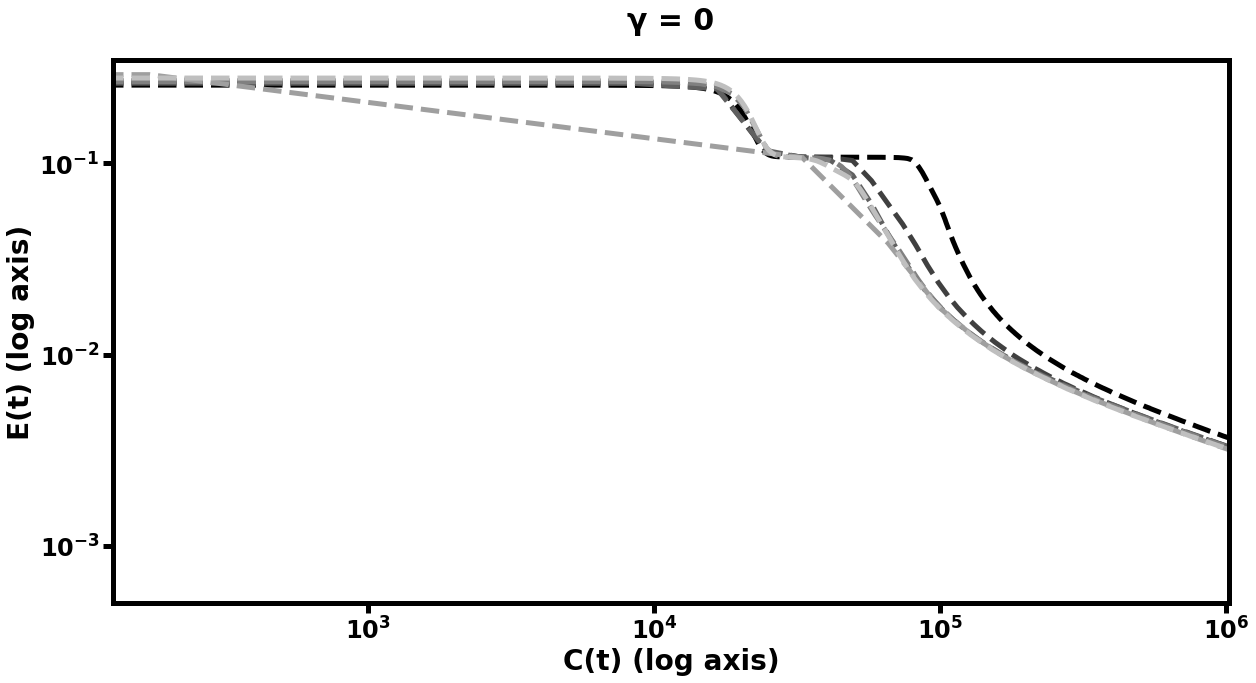} \\
	\includegraphics[width=\tplotfactor\linewidth]{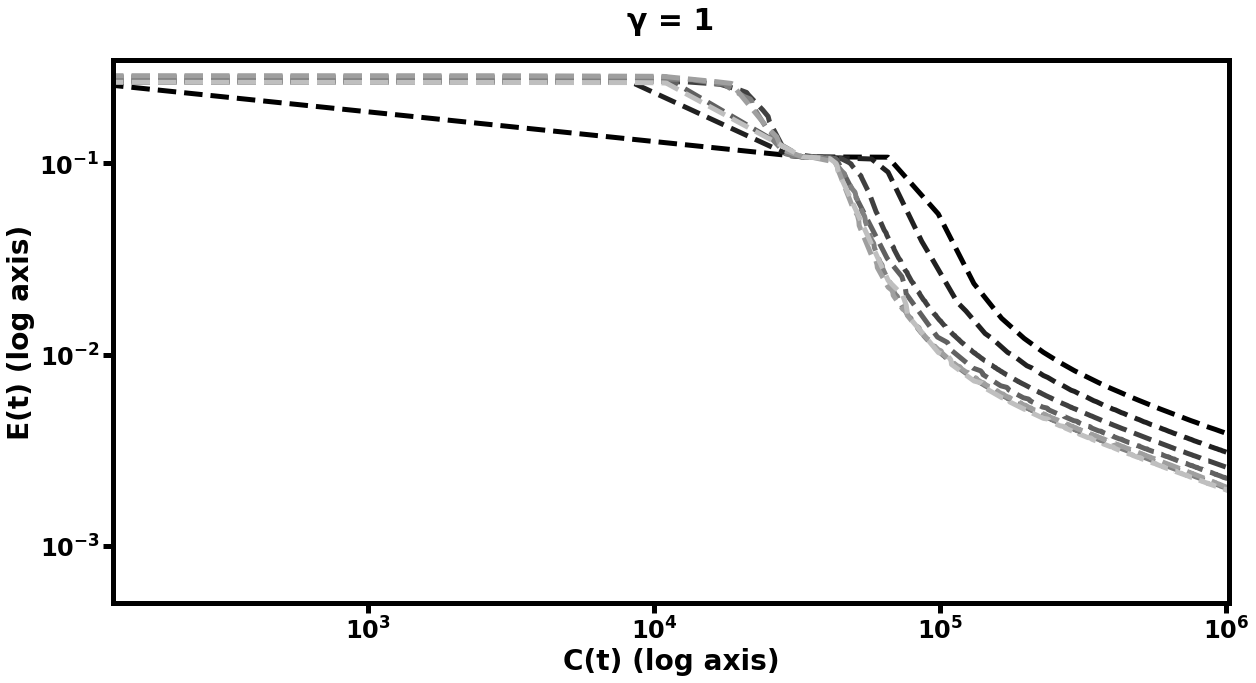} \\
	\includegraphics[width=\tplotfactor\linewidth]{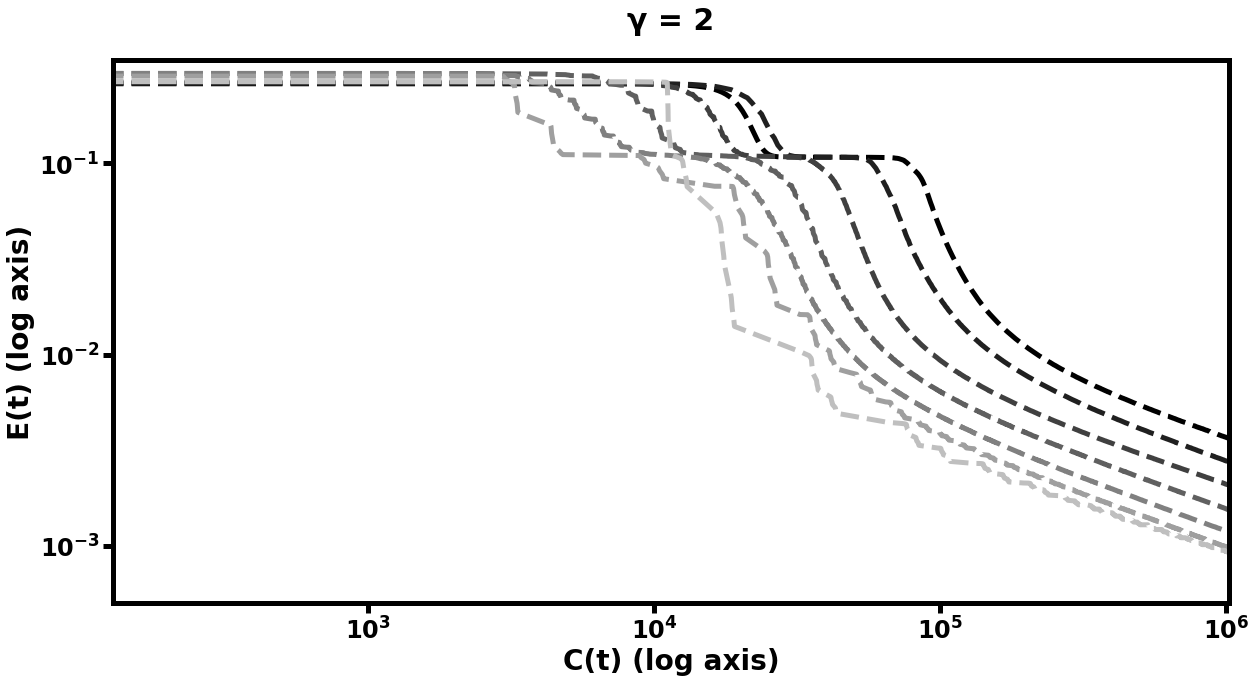} \\
	\includegraphics[width=\tplotfactor\linewidth]{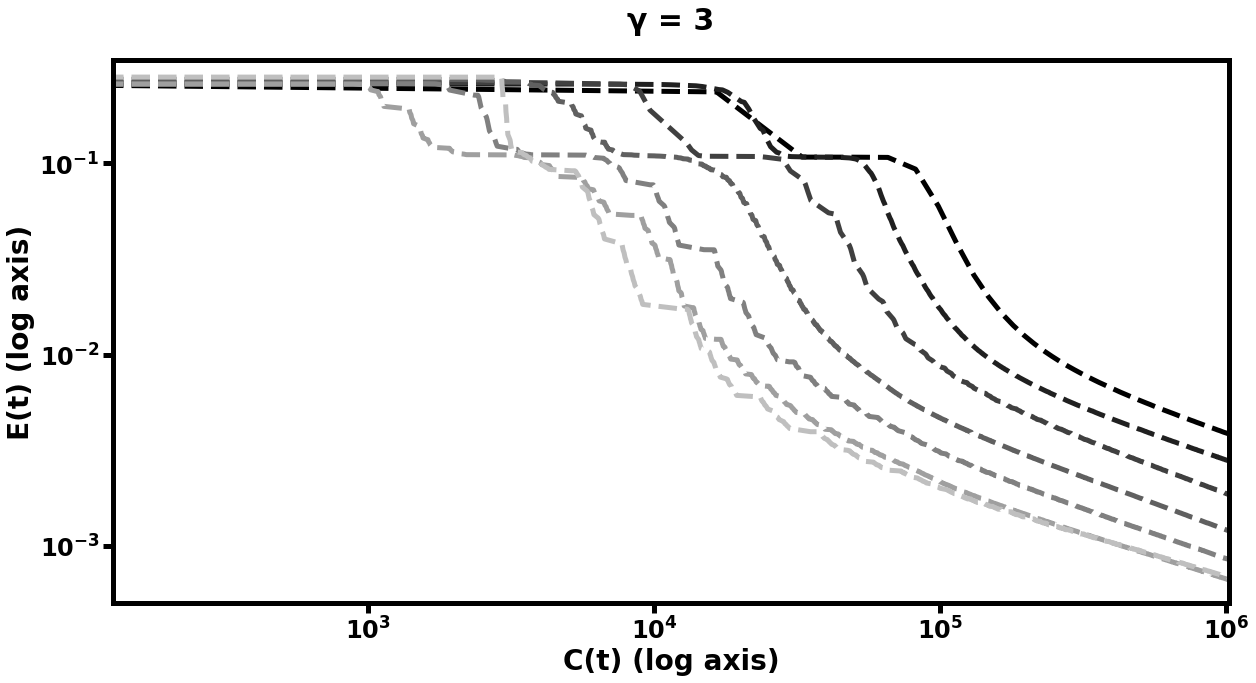} \\
\end{minipage}
\begin{minipage}{.47\linewidth}
	\includegraphics[width=\tplotfactor\linewidth]{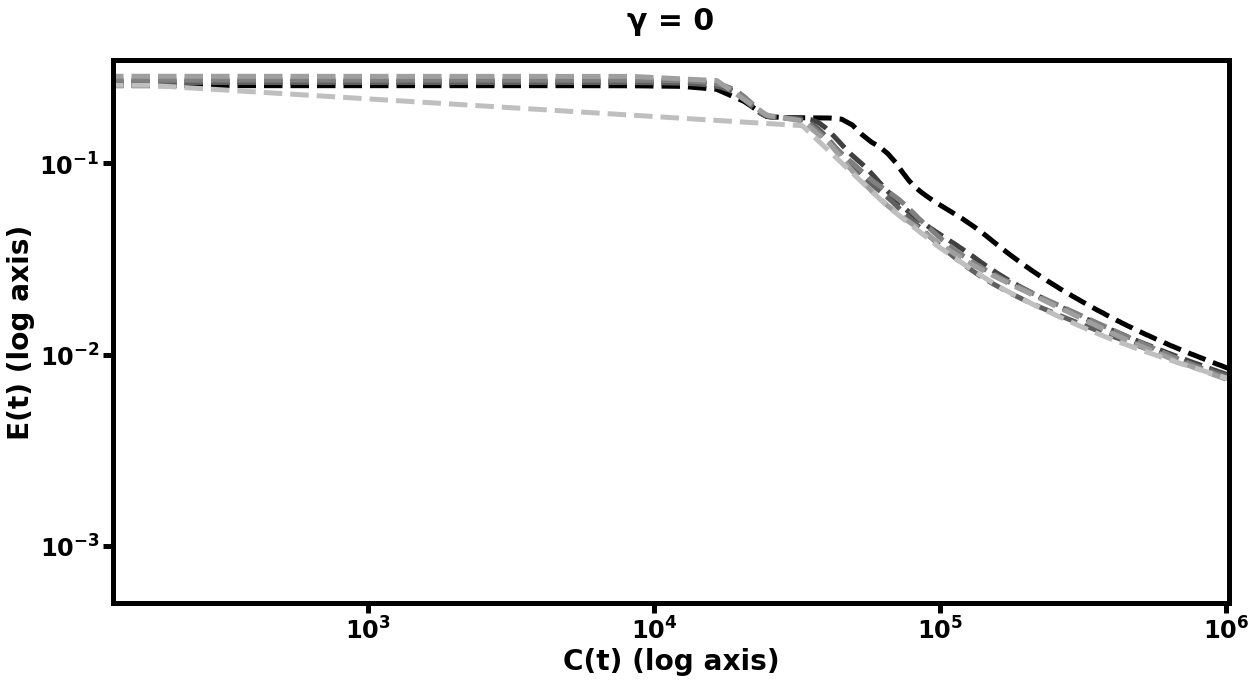} \\
	\includegraphics[width=\tplotfactor\linewidth]{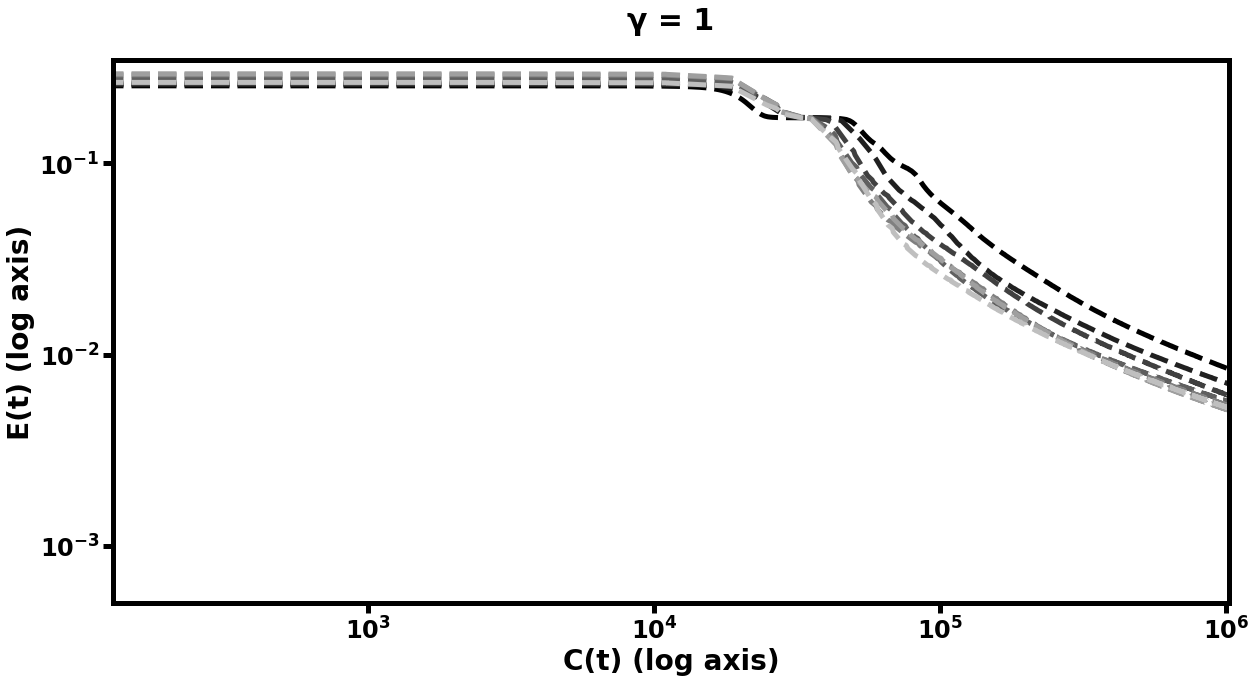} \\
	\includegraphics[width=\tplotfactor\linewidth]{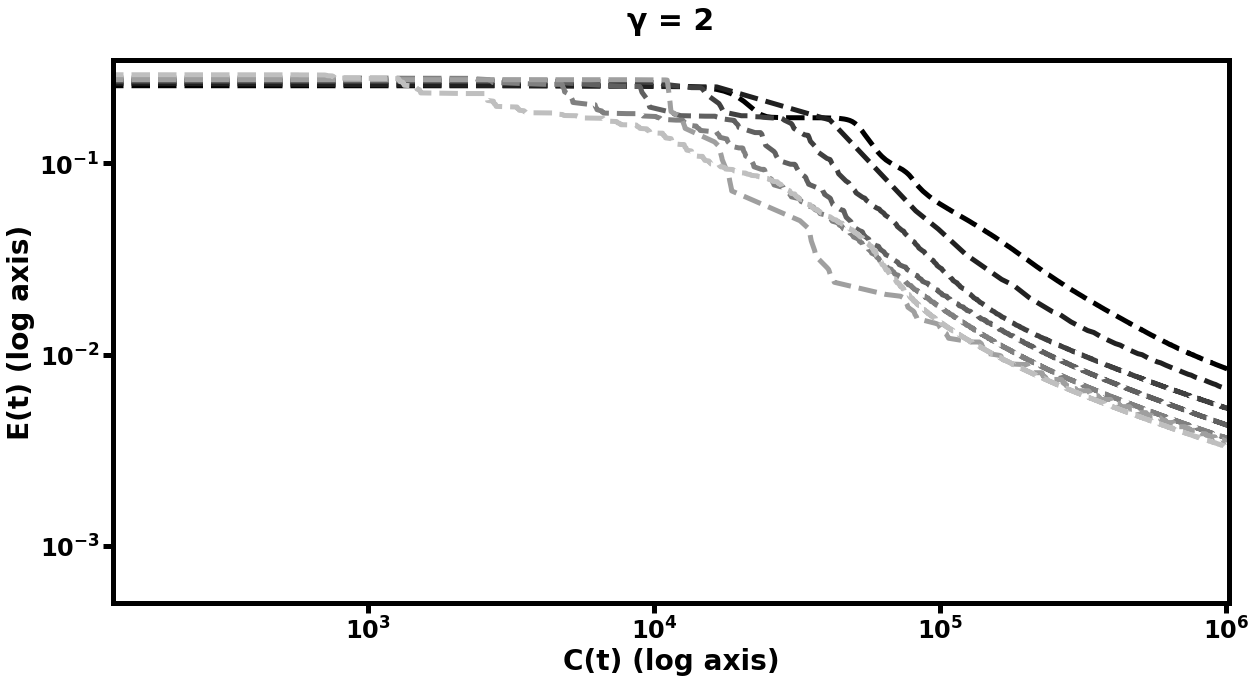} \\
	\includegraphics[width=\tplotfactor\linewidth]{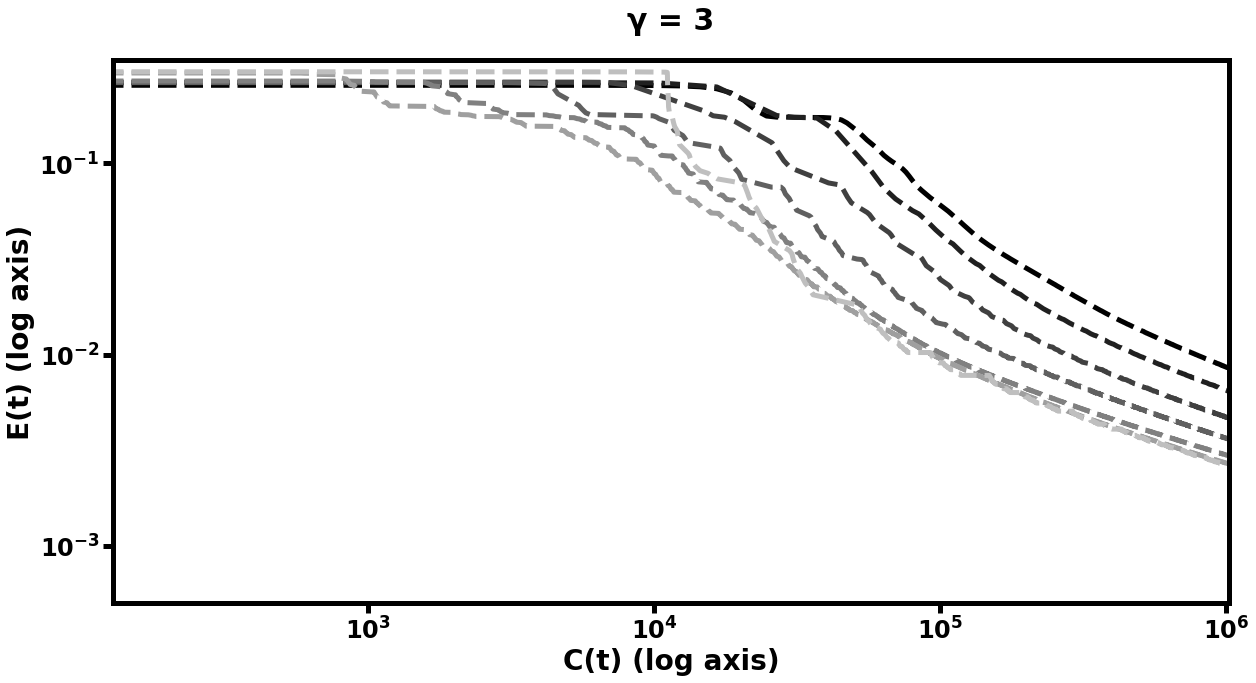} \\
\end{minipage} \\
\caption{Log-log plots of accuracy $E(t)$ as a function of training cost $C(t)$ attained by a variety of hierarchical neural networks training on a simple machine vision task, demonstrating that deeper hierarchies with more mutligrid behavior learn more rapidly. Plots are ordered from top to bottom in increasing depth of recursion parameter $\gamma$; left plots are the single-object experiments and right plots are the double-object experiments. Within each plot, different curves represent different values of the depth of hierarchy, from $L=6$ (lightest) to $L=0$ (darkest). Each line is the best run for that pair $(L, \gamma)$ over all choices of $k$ (number of smoothing steps at each level) in $\{1,2,4,8,16,32,64,128\}$. }
\end{figure}

\begin{table}[h]%
{\Large
\resizebox{\columnwidth}{!}{%
\begin{tabular}{|r || S[table-parse-only] | S[table-parse-only] | S[table-parse-only] | l |}
\hline
 & {Best MsANN} & {Worst MsANN} & {Default} & Best MsANN params\\
\hline
Final MSE & 6.612e-04 & 4.431e-03 & 3.654e-03 & $(\gamma = 3, L = 5, k = 004)$\\
\hline
Cost to $\frac{1}{10}$ MSE & 7.342e+03 & 1.640e+05 & 1.266e+05 & $(\gamma = 3, L = 6, k = 004)$\\
\hline
\end{tabular}%

}
\caption{Best performance (on validation dataset for the one-object autoencoding task) by any combination of parameters in our sweep over values for $\gamma$ (recursion constant), $L$ (depth of network), and $k$ (number of batches processed at each visit to each level). We report the final Mean-Squared Error for both the best and worst combination of these parameters, as well as for default training. We also report the best combination of parameters. The second row indicates the cost $C(t)$ necessary to train each model to $\frac{1}{10}$ of the error at which it began. The best MsANN network reaches this threshhold in an order of magnitude less cost, and its final error is roughly half that of the default model, demonstrating clear improvement over training without multigrid.}
\label{tbl:oneobj}
}
\end{table}

\subsubsection{Double-Object Autoencoder}
\label{subsub:2obj}
We repeat the above experiment with a slightly more difficult machine vision task - the network must learn to de-noise an image with two  (non-overlapping) `objects' in the visual field. We use the same network structure and training procedure, and note that we see again (plotted in Figure \ref{fig:autofig} and summarized in Table \ref{tbl:twoobj}) that the hierarchical model is more efficient, reaching lower error in the same amount of computational cost $C(t)$. The multigrid neural networks again typically learn much more rapidly than the non-multigrid models.

\begin{table}[h]%
\resizebox{\columnwidth}{!}{%
\begin{tabular}{|r || S[table-parse-only] | S[table-parse-only] | S[table-parse-only] | l |}
\hline
 & {Best MsANN} & {Worst MsANN} & {Default} & Best MsANN params\\
\hline
Final MSE & 2.576e-03 & 8.998e-03 & 8.816e-03 & $(\gamma = 3, L = 6, k = 002)$\\
\hline
Cost to $\frac{1}{10}$ MSE & 2.433e+04 & 2.623e+05 & 2.216e+05 & $(\gamma = 3, L = 6, k = 016)$\\
\hline
\end{tabular}%

}
\caption{Best performance (on validation dataset for the two-object autoencoding task). Again the MsANN network demonstrates performance and accuracy gains over neural network training alone. See Table \ref{tbl:oneobj}.}
\label{tbl:twoobj}
\end{table}

\subsection{MNIST}
\label{subsec:MNIST}
To supplement the above synthetic experiments with one using real-world data, we perform the same experiment with an autoencoder for the MNIST handwritten digit dataset \cite{lecun1998gradient, lecun2010mnist}. In this case, rather than the usual MNIST classification task, we use an autoencoder to map the MNIST images into a lower-dimensional $(d=128)$ space with good reconstruction. We use the same network structure as in the 1D vision example; also as in that example, each network in the hierarchy is constructed of fully connected layers with bias terms and sigmoid activation, and smoothing steps are performed with RMSProp \cite{hinton2012neural} with learning rate 0.0005. The only difference is that in this example we do not add noise to the input images, since the dataset is larger by two orders of magnitude.

In this experiment, we see (in Figure \ref{fig:mnistfig} and Table \ref{tbl:mnist}) similar improvement in efficiency. Table \ref{tbl:mnist} summarizes these results: the best multilevel models learned more rapidly and achieved lower error than their single-level counterparts, whereas the worst multilevel models performed on par with the default model.  Because the MNIST data is comprised of 2D images, we tried using $P$ matrices which were the optima of prolongation problems between grids of the appropriate sizes, in addition to the same 1D $P$ used in the prior two experiments. The difference in performance between these two choices of underlying structure for the prolongation maps can be seen in Figure \ref{fig:mnistfig}. With either approach, we see similar results to the synthetic data experiment, in that more training steps at the coarser layers results in improved learning performance of the finer networks in the hierarchy. However, the matrices optimized for 2D prolongation perform marginally better than their 1D cousins, - in particular, the multigrid hierarchy with 2D prolongations took 60\% of the computational cost to reduce its error to $\frac{1}{10}$ of its original value, as compared to the 1D version. We explore the effect of varying the strategy used to pick $P$ in Subsection \ref{subsec:p_choice}.

\begin{figure}
\centering
\label{fig:mnistfig}
\begin{minipage}{.47\linewidth}
	\includegraphics[width=\tplotfactor\linewidth]{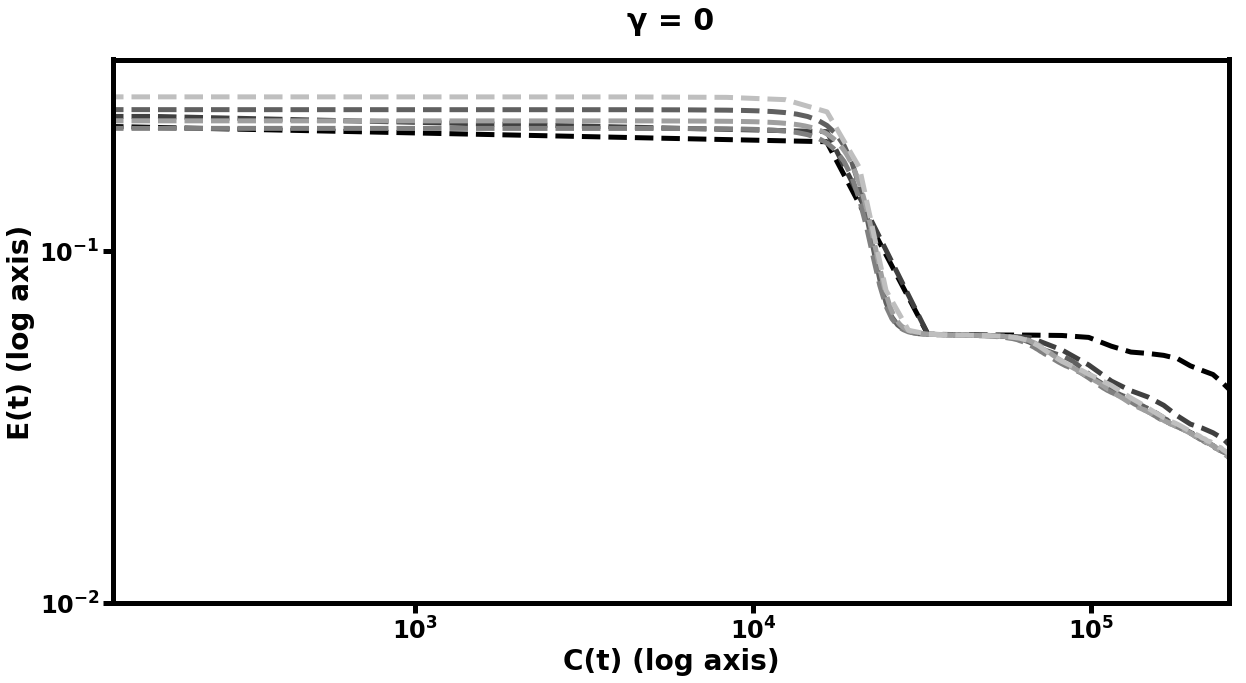} \\
	\includegraphics[width=\tplotfactor\linewidth]{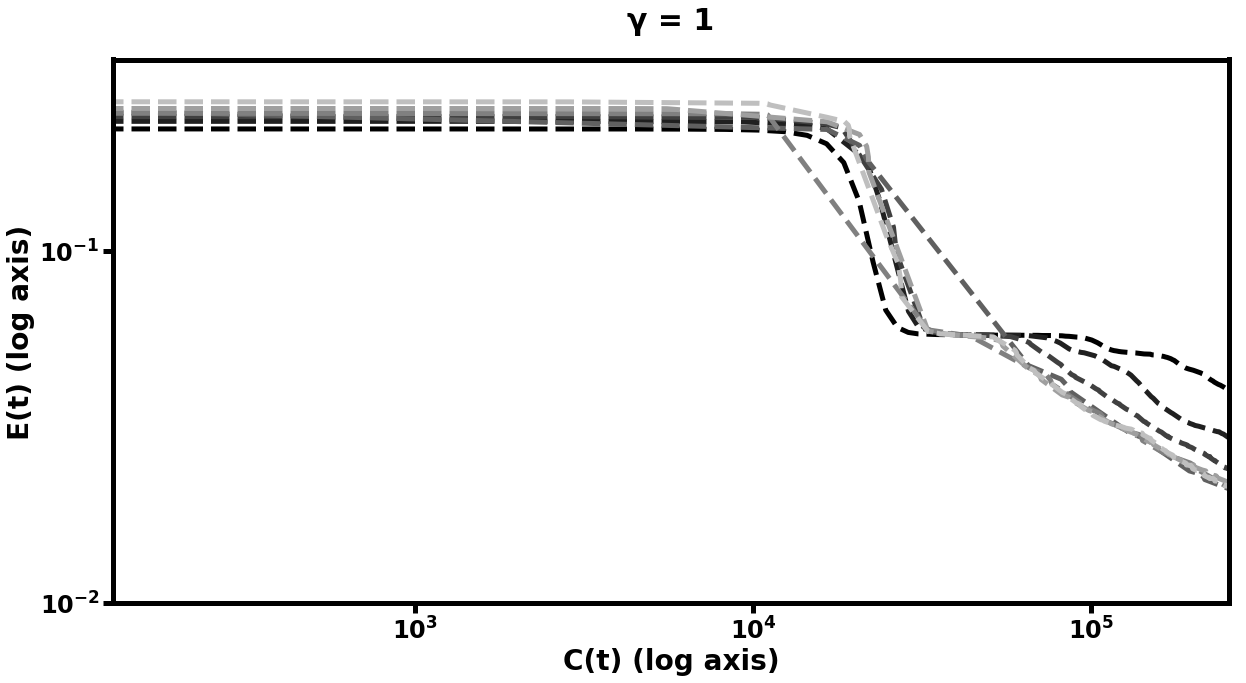} \\
	\includegraphics[width=\tplotfactor\linewidth]{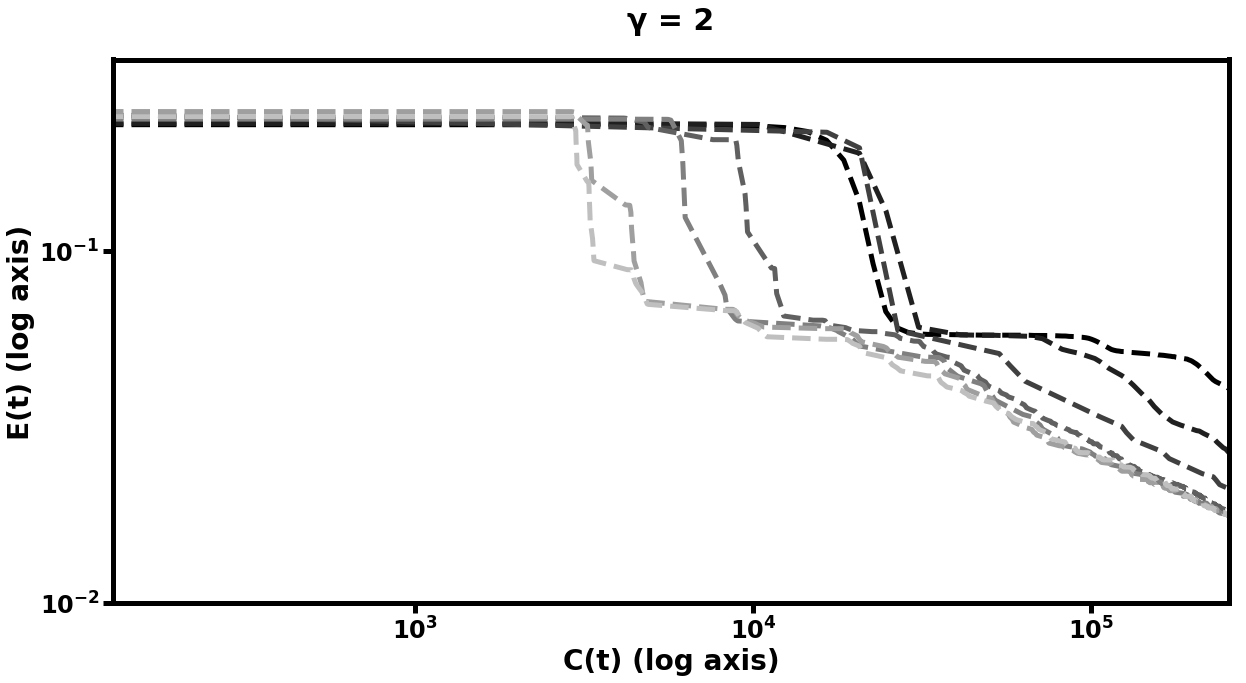} \\
	\includegraphics[width=\tplotfactor\linewidth]{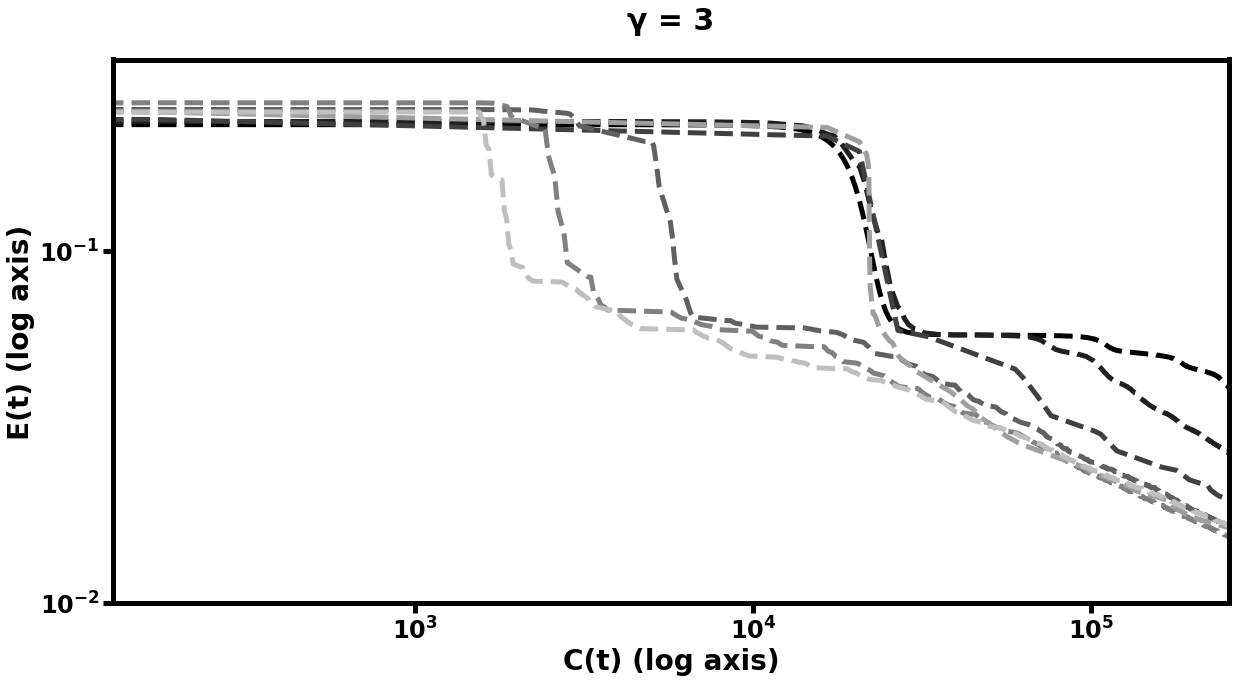} \\
\end{minipage}
\begin{minipage}{.47\linewidth}
	\includegraphics[width=\tplotfactor\linewidth]{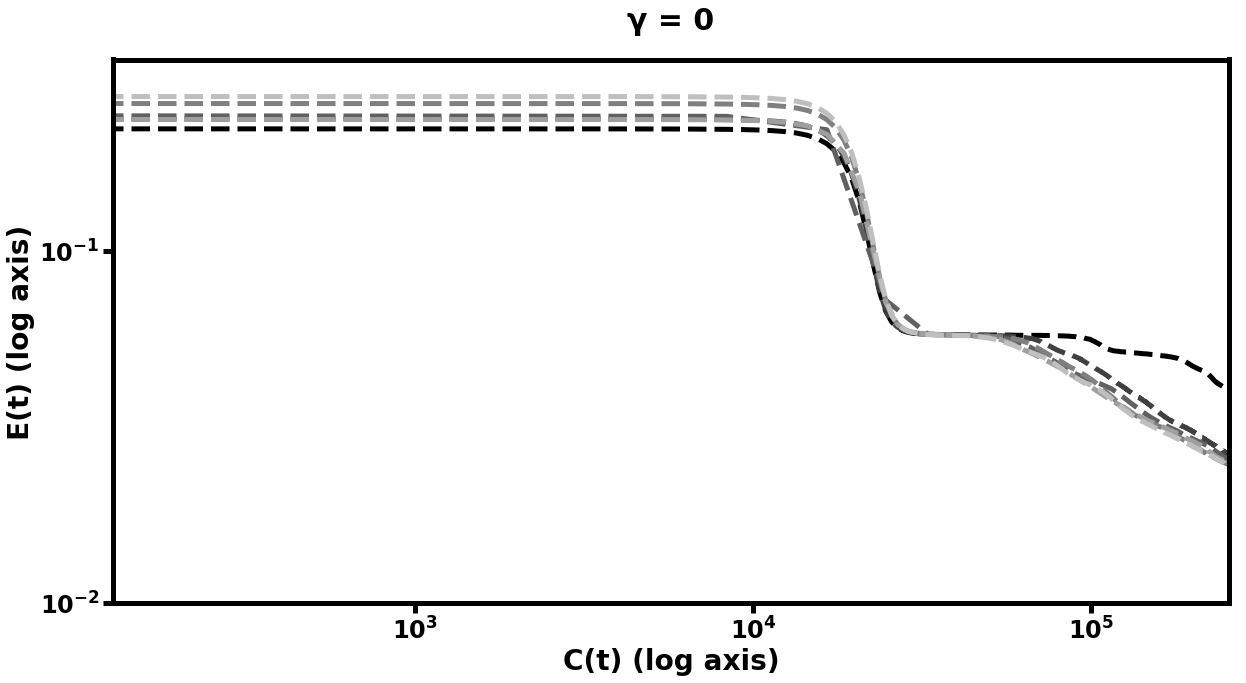} \\
	\includegraphics[width=\tplotfactor\linewidth]{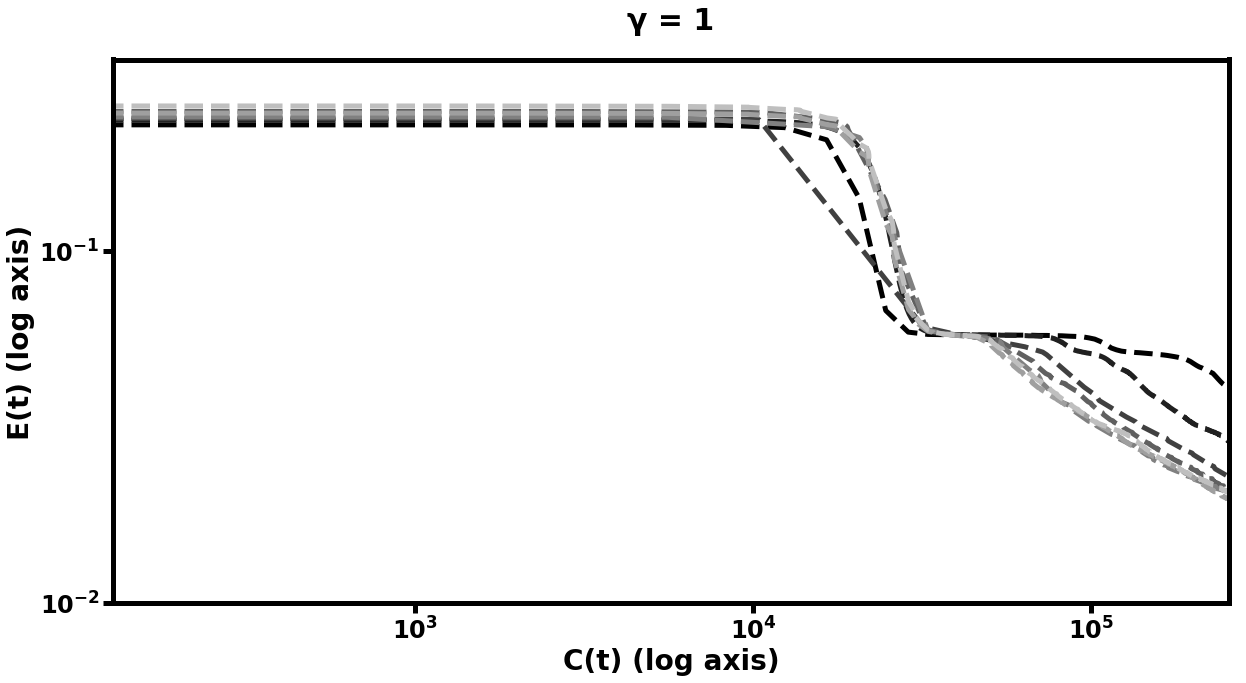} \\
	\includegraphics[width=\tplotfactor\linewidth]{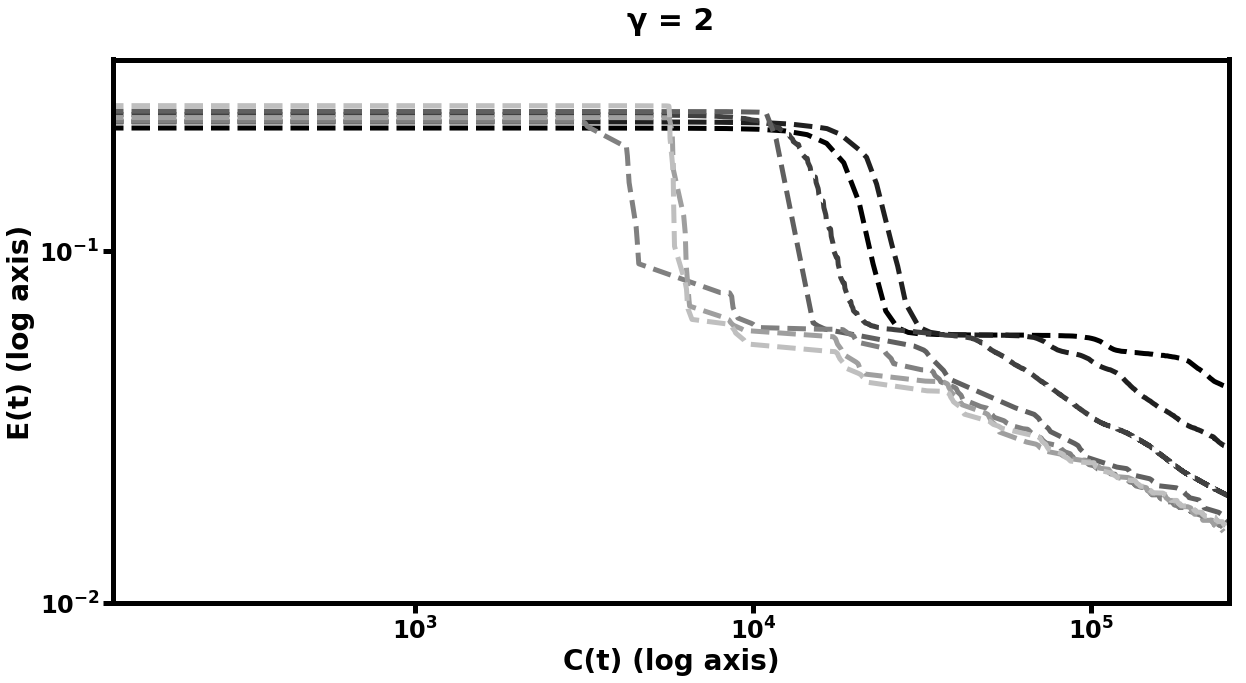} \\
	\includegraphics[width=\tplotfactor\linewidth]{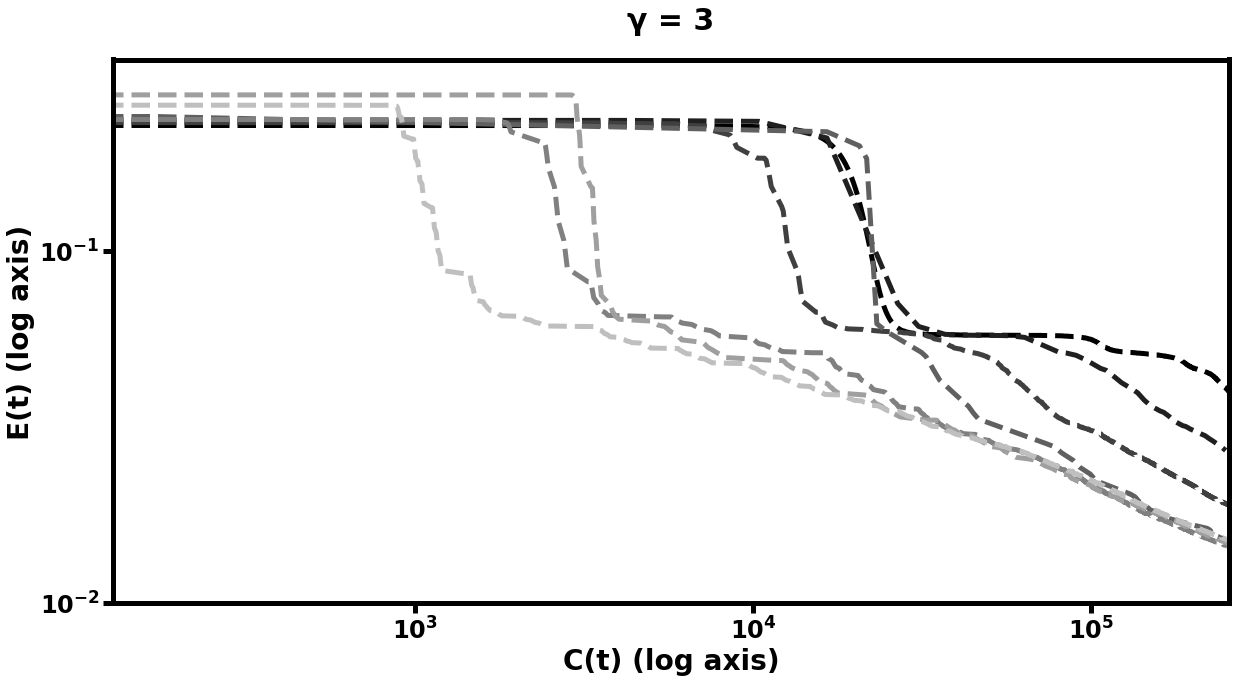} \\
\end{minipage} \\
\caption{Log-log plots of mean-squared error $E(t)$ on MNIST autoencoding task as a function of computational cost $C(t)$; the left plots represent multigrid performed with path graph prolongations for each layer while the right plots used grid-based prolongation. While both approaches show gains over default learning in both speed of learning and final error value, the one which respects the spatial structure of the input data improves more rapidly. Subplot explanations are the same as in Figure \ref{fig:autofig}.}
\end{figure}

\begin{table}[h]%
\resizebox{\columnwidth}{!}{%
\begin{tabular}{|r || S[table-parse-only] | S[table-parse-only] | S[table-parse-only] | l |}
\hline
\multicolumn{5}{|c|}{Path-Based $P$ Matrices} \\
\hline
 & {Best MsANN} & {Worst MsANN} & {Default} & Best MsANN params\\
\hline
Final MSE & 1.547e-02 & 4.605e-02 & 4.171e-02 & $(\gamma = 3, L = 4, k = 008)$\\
\hline
Cost to $\frac{1}{10}$ MSE & 7.207e+04 & N/A & N/A & $(\gamma = 3, L = 5, k = 032)$\\
\hline
\multicolumn{5}{|c|}{Grid-Based $P$ Matrices} \\
\hline
 & {Best MsANN} & {Worst MsANN} & {Default} & Best MsANN params\\
\hline
Final MSE & 1.436e-02 & 4.620e-02 & 4.132e-02 & $(\gamma = 3, L = 4, k = 002)$\\
\hline
Cost to $\frac{1}{10}$ MSE & 5.095e+04 & N/A & N/A & $(\gamma = 3, L = 6, k = 128)$\\
\hline
\end{tabular}%

}
\caption{Best performance (on validation dataset for the MNIST autoencoding task). See Table \ref{tbl:oneobj}. Upper section represents scores attained by a MsANN with path-based prolongation, lower section represents grid-based prolongation. Entries marked N/A did not reach $\frac{1}{10}$ of their initial error during training.}
\label{tbl:mnist}
\end{table}

\subsection{Experiments of Choice of $P$}
\label{subsec:p_choice}
To further explore the role of the structure of $P$ in these machine learning models, we compare the performance of several MsANN models with $P$ generated according to various strategies. Our initial experiment on the MNIST dataset used the exact same hierarchical network structure and prolongation/restriction operators as the example with 1D data, and yielded marginal computational benefit. We were thus motivated to try this learning task with prolongations which are designed for for 2D grid-based model architectures, as well as trying unstructured (random orthogonal) matrices as a baseline. More precisely, our 1D experiments used $P$ matrices resembling those in column 3 of the ``Cycle Graphs" section of Figure \ref{fig:pspeciesfig}. We instead, for the MNIST task, used $P$ matrices like those in column $6$ of the ``Grid Graphs'' section of the same figure. In Figure \ref{fig:pcomp_mnist}, we illustrate the difference in these choices for the MNIST training task, with the same choice of multigrid training parameters: $(L = 6, \gamma = 3, k = 1)$. We compare the following strategies for generating $P$:
\begin{enumerate}
    \item \label{misc:plist_elem_0} As local optima of a prolongation problem between 1D grids, with periodic boundary conditions;
    \item \label{misc:plist_elem} As local optima of a prolongation problem between 2D grids, with periodic boundary conditions;
    \item \label{misc:plist_elem_2} As in \ref{misc:plist_elem}, but shuffled along the first index of the array.
\end{enumerate}
Strategy \ref{misc:plist_elem_2} was chosen to provide the same degree of connectivity between each coarse variable and its related fine variables as strategy \ref{misc:plist_elem}, but in random order i.e. connected in a way which is unrelated to the 2D correlation between neighboring pixels.
We see in Figure \ref{fig:pcomp_mnist} that the two strategies utilizing local optima outperform both the randomized strategy and default training (training only the finest scale). Furthermore, strategy \ref{misc:plist_elem} outperforms strategy \ref{misc:plist_elem_0}, although the latter eventually catches up at the end of training, when coarse-scale weight training has diminishing marginal returns. The random strategy is initially on par with the two optimized ones (we speculate that this is due to the ability to affect many fine-scale variables at once, even in random order, which may make the gradient direction easier to travel), but eventually falls behind, at times being less efficient than default training. We leave for further work the question of whether there are choices of prolongation problem which are even more efficient for this machine learning task. We also compare all of the preceeding models to a model which has the same structure as a MsANN model (a hierarchy of coarsened variables with $\text{Pro}$ and $\text{Res}$ operators between them), but which was trained by training all variables in the model simultaneously. This model performs on par with the default model, illustrating the need for the multilevel training schedule dictated by the choice of $\gamma$.

\begin{figure}
\begin{minipage}[c]{.64\linewidth}
    \includegraphics[width=\linewidth]{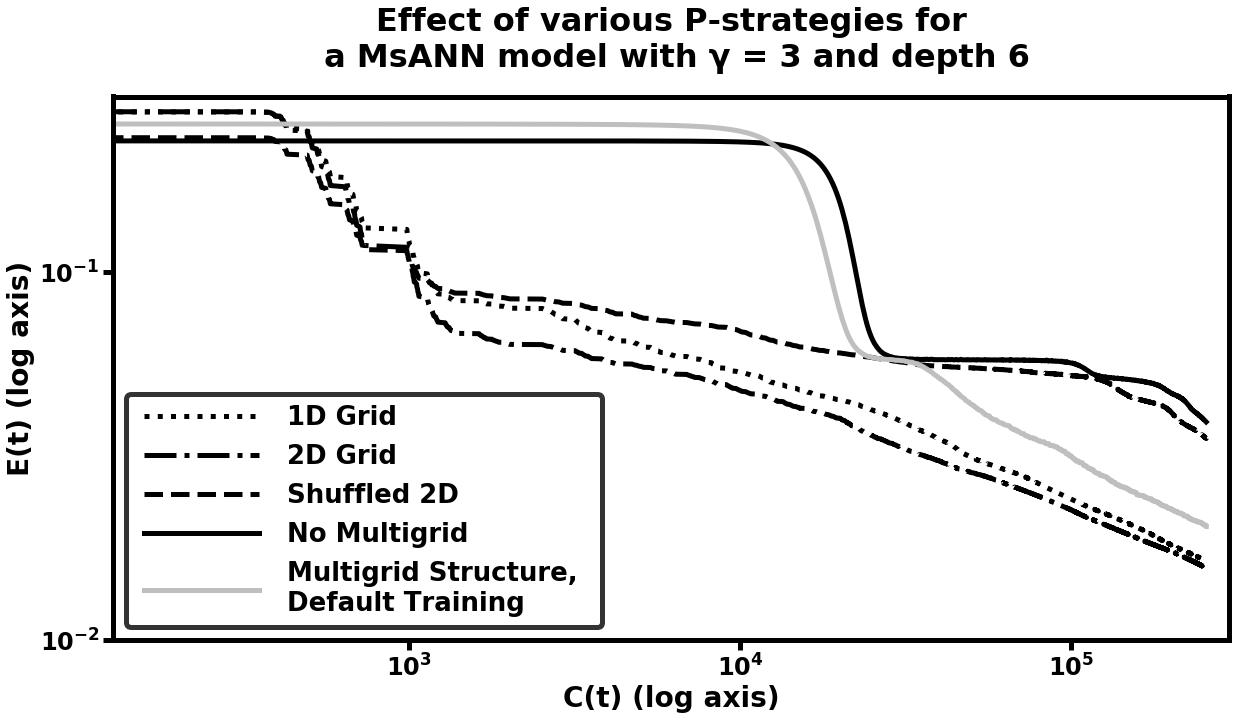}
\end{minipage}
\begin{minipage}[c]{.33\linewidth}
    \caption{Comparison of several choices of $\text{Pro}$ and $\text{Res}$ operators for a Multiscale Neural Network training experiment, on MNIST data. Two choices for $P$ which are local optima of prolongation problems demonstrate more efficient training than default, while two strategies perform worse: multigrid training with random $P$ matrices, and training all varibles in the hierarchy simultaneously.}
    \label{fig:pcomp_mnist}
\end{minipage}
\end{figure}

\subsection{Summary}
We see uniform improvement (as the parameters $L$ and $\gamma$ are increased) in the rate of neural network learning when models are stacked in the type of multiscale hierarchy we define in equations \ref{eqn:weightpro} and \ref{eqn:biaspro}, despite the diversity of machine learning tasks we examine. Furthermore, this improvement is marked: the hierarchical models both learn more rapidly than training without multigrid and have final error lower than the default model. In many of our test cases, the hierarchical models reached the same level of MSE as the default in more than an order of magnitude fewer training examples, and continued to improve, surpassing the final level of error reached by the default network. Even in the worst case, our hierarchical model structure performed on par with neural networks which did not incorporate our weight prolongation and restriction operators. We leave the question of finding optimal $(L, \gamma, k)$ for future work - see Section \ref{sec:future} for further discussion. Finally, we note that the model(s) in the experiments presented in section \ref{subsec:auto} were essentially the same MsANN models (same set of $L, \gamma, k$ and same set of $P$ matrices), and showed similar performance gains on two different machine vision problems, indicating that it may be possible to develop general MsANN model-creation procedures that are applicable to a variety of problems (rather than needing to be hand-tuned). 

\section{Upper Bounds for Diffusion Term}
\label{sec:theoryshortsec}
In this section, we consider two theoretical concerns: 
\begin{enumerate}
\item Invariance in Frobenius norm of diffusion term solutions under transformation to a spectral basis; and
\item Decoupling a prolongation problem between graph products into a sum of prolongation problems of the two sets of graph factors.
\end{enumerate}
We will here rely heavily on various properties of the Kronecker sum and product of matrices which may be found in \cite{hogben2006handbook}, Section 11.4. 
\subsection{Invariance of objective function evaluation of \textit{P} under eigenspace transformation}
\label{subsec:freqproof}
For the purpose of the calculations in this section, we restrict ourselves to the ``diffusion'' term of our objective function \ref{eqn:objfunction} (the term which coerces two diffusion processes to agree), which we will write as 
\begin{align}
\label{eqn:distancedefn}
D_{P,\alpha} \left( G_1, G_2 \right) &= \left| \left| \frac{1}{\sqrt{\alpha}} P L_1 -  \sqrt{\alpha} L_2 P \right| \right|_F .
\end{align}
Because $L_1$ and $L_2$ are each real and symmetric, they may both be diagonalized as  $L_i = U_i \Lambda_i U_i^T$ where $U_i$ is a rotation matrix and $\Lambda_i$ is a diagonal matrix with the eigenvalues of $L_i$ on the diagonal. Substituting into \ref{eqn:distancedefn}, and letting $\tilde{P} = U_2^T P U_1$, we have
\begin{align}
D_{P,\alpha} \left( G_1, G_2 \right) &= \left| \left| \frac{1}{\sqrt{\alpha}} P L_1 -  \sqrt{\alpha} L_2 P \right| \right|_F \nonumber \\
&= \left| \left| \frac{1}{\sqrt{\alpha}} P U_1 \Lambda_1 U_1^T -  \sqrt{\alpha} U_2 \Lambda_2 U_2^T P \right| \right|_F \nonumber \\
&= \left| \left| \frac{1}{\sqrt{\alpha}} \left( U_2^T P U_1 \right) \Lambda_1 -  \sqrt{\alpha} \Lambda_2 \left( U_2^T P U_1 \right) \right| \right|_F \nonumber \\
&= \left| \left| \frac{1}{\sqrt{\alpha}} \tilde{P} \Lambda_1 -  \sqrt{\alpha} \Lambda_2 \tilde{P} \right| \right|_F \label{eqn:spectral_form}
\end{align}
where $\tilde{P}$ is an orthogonal matrix $\tilde{P}^T \tilde{P} = I$ if and only if $P$ is. Since the Frobenius norm is invariant under multiplication by rotation matrices, \ref{eqn:spectral_form} is a re-formulation of our original Laplacian matrix objective function in terms of the spectra of the two graphs. 
Optimization of this modified form of the objective function subject to orthogonality constraints on $P$ is upper-bounded by optimization over matchings of eigenvalues: for any fixed $\alpha$ the eigenvalue-matching problem has the same objective function, but our optimization is over all real valued orthogonal $P$. The orthogonality constraint is a relaxed version of the constraints on matching problems (Equation \ref{eqn:matchingconstraints}) discussed in subsection \ref{subsub:definitions}, since matching matrices M are also orthogonal $(M^T M = I)$. Many algorithms exist for solving the inner partial and 0-1 constrained minimum-cost assignment problems, such as the Munkres algorithm \cite{munkres1957algorithms} (also in subsection \ref{subsub:definitions}). 

We note three corollaries of the above argument. Namely, because the Frobenius norm is invariant under the mapping to and from eigenspace:
\begin{enumerate}
\item Optimal or near-optimal $\tilde{P}$ in eigenvalue-space maintain their optimality through the mapping $U_2 \cdot U_1^T$ back to graph-space.
\item Solutions which are within $\epsilon$ of the optimum in $\tilde{P}$-space are also within $\epsilon$ of the optimum in $P$-space; and
\item More precisely, if they exist, zero-cost eigenvalue matchings correspond exactly with zero-cost $P$. 
\end{enumerate}

A natural next question would be why it might be worthwhile to work in the original graph-space, rather than always optimizing this simpler eigenvalue-matching problem instead. In many cases (path graphs, cycle graphs) the spectrum of a member $G_l$ of a graph lineage is a subset of that of $G_{l+1}$, guaranteeing that zero-cost eigenvalue matchings (and thus, by the argument above, prolongations with zero diffusion cost) exist. However, when this is not the case, the above argument only upper bounds the true distance, since the matching problem constraints are more strict. Thus, numerical optimization over $P$, with orthogonality constraints only, may find a better bound on $D^{P,\alpha} \left( G_{l},  G_{l+1} \right)$.
\subsection{Decomposing Graph Product Prolongations}
\label{subsec:deompose}
We next consider the problem of finding optimal prolongations between two graphs $\mathbf{G}_\Box^{(1)} = G^{(1)}_1 \Box G^{(2)}_1$ and $\mathbf{G}_\Box^{(2)} = G^{(1)}_2 \Box G^{(2)}_2$ when optimal prolongations are known between $G^{(1)}_1$ and $G^{(1)}_2$, and $G^{(2)}_1$ and $G^{(2)}_2$. We show that under some reasonable assumptions, these two prolongation optimizations decouple - we may thus solve them separately and combine the solutions to obtain the optimal prolongations between the two product graphs. 

From the definition of graph box product, we have 
\begin{align*}
L_\Box^{(j)} &= L(G_1^{(j)} \Box G_2^{(j)}) \\
		     &= A(G_1^{(j)} \Box G_2^{(j)}) - D(G_1^{(j)} \Box G_2^{(j)}) \\
             &= \left( A(G_1^{(j)}) \otimes I_2^{(j)}  + I_1^{(j)} \otimes A(G_2^{(j)}) \right) - \left( D(G_1^{(j)}) \otimes I_2^{(j)}  + I_1^{(j)} \otimes D(G_2^{(j)}) \right) \\
             &= \left( A(G_1^{(j)}) \otimes I_2^{(j)} -D(G_1^{(j)}) \otimes I_2^{(j)} \right) - \left(I_1^{(j)} \otimes A(G_2^{(j)})  - I_1^{(j)} \otimes D(G_2^{(j)}) \right) \\
             &= (L_1^{(j)} \otimes I_2^{(j)})  + (I_1^{(j)} \otimes L_2^{(j)}) \\
             &= L(G_1^{(j)}) \oplus L(G_2^{(j)})
\end{align*}%
where $\oplus$ is the Kronecker sum of matrices as previously defined. See \cite{fiedler1973algebraic}, Item 3.4 for more details on Laplacians of graph products. We calculate

\begin{align*}
D^{P,\alpha} \left( G_\Box^{(1)},  G_\Box^{(2)} \right) &= \left| \left| \frac{1}{\sqrt{\alpha}} P L_\Box^{(1)} -  \sqrt{\alpha} L_\Box^{(2)} P \right| \right|_F\\
&= \left| \left| \frac{1}{\sqrt{\alpha}} P \left( \left( L_1^{(1)} \otimes I_2^{(1)} \right) + \left(I_1^{(1)} \otimes L_2^{(1)}\right) \right) \right. \right. \\
& \left. \left. \qquad -  \sqrt{\alpha} \left( \left( L_1^{(2)} \otimes I_2^{(2)}\right) + \left(I_1^{(2)} \otimes L_2^{(2)} \right) \right) P \right| \right|_F \\
&= \left| \left| \left( \frac{1}{\sqrt{\alpha}} P \left(L_1^{(1)} \otimes I_2^{(1)}\right) - \sqrt{\alpha}\left(L_1^{(2)} \otimes I_2^{(2)}\right)  P \right) \right. \right. \\
& \left. \left. \qquad + 
\left( \frac{1}{\sqrt{\alpha}} P \left( I_1^{(1)} \otimes L_2^{(1)} \right) - \sqrt{\alpha} \left( I_1^{(2)} \otimes L_2^{(2)} \right) P\right) \right| \right|_F \\
\intertext{Now we try out the assumption that $P = P_1 \otimes P_2$, which restricts the search space over $P$ and may increase the objective function:}
D^{P,\alpha} \left( G_\Box^{(1)},  G_\Box^{(2)} \right) &= \left| \left| \left[ \frac{1}{\sqrt{\alpha}} \left( P_1 \otimes P_2 \right)\left(L_1^{(1)} \otimes I_2^{(1)}\right)  \right. \right. \right. \\
& \qquad \qquad - \left. \sqrt{\alpha}\left(L_1^{(2)} \otimes I_2^{(2)}\right)  \left( P_1 \otimes P_2 \right) \right] \\
& \qquad + 
\left[ \frac{1}{\sqrt{\alpha}} \left( P_1 \otimes P_2 \right) \left( I_1^{(1)} \otimes L_2^{(1)} \right) \right. \\
& \left. \left. \left. \qquad \qquad - \sqrt{\alpha} \left( I_1^{(2)} \otimes L_2^{(2)} \right) \left( P_1 \otimes P_2 \right) \right] \right| \right|_F \\
&= \left| \left| \left( \frac{1}{\sqrt{\alpha}} \left( P_1 L_1^{(1)} \otimes P_2 \right) - \sqrt{\alpha}\left(L_1^{(2)}  P_1 \otimes P_2 \right) \right) \right. \right. \\
& \left. \left. \qquad + 
\left( \frac{1}{\sqrt{\alpha}} \left( P_1 \otimes P_2  L_2^{(1)} \right) - \sqrt{\alpha} \left( P_1 \otimes L_2^{(2)} P_2 \right) \right) \right| \right|_F \\
&= \left| \left| \left(  \left( \frac{1}{\sqrt{\alpha}} P_1 L_1^{(1)}  - \sqrt{\alpha} L_1^{(2)}  P_1 \right) \otimes P_2 \right) \right. \right. \\
& \left. \left. \qquad + 
\left(  P_1 \otimes \left( \frac{1}{\sqrt{\alpha}} P_2  L_2^{(1)} - \sqrt{\alpha} L_2^{(2)} P_2 \right) \right) \right| \right|_F
\intertext{Since $|| A + B ||_F \leq ||A||_F + ||B||_F$,}
&\leq \left| \left| \left(  \left( \frac{1}{\sqrt{\alpha}} P_1 L_1^{(1)}  - \sqrt{\alpha} L_1^{(2)}  P_1 \right) \otimes P_2 \right) \right| \right| \\
& \qquad + 
\left| \left| \left(  P_1 \otimes \left( \frac{1}{\sqrt{\alpha}} P_2  L_2^{(1)} - \sqrt{\alpha} L_2^{(2)} P_2 \right) \right) \right| \right|_F \\
&= \left| \left| \frac{1}{\sqrt{\alpha}} P_1 L_1^{(1)}  - \sqrt{\alpha} L_1^{(2)}  P_1 \right| \right|_F \left| \left| P_2 \right| \right|_F \\
& \qquad + 
\left| \left| P_1 \right| \right|_F \left| \left| \frac{1}{\sqrt{\alpha}} P_2  L_2^{(1)} - \sqrt{\alpha} L_2^{(2)} P_2 \right| \right|_F , \\
\intertext{Thus assuming $P = P_1 \otimes P_2$}
D^{P,\alpha} \left( G_\Box^{(1)},  G_\Box^{(2)} \right) &\leq \left| \left| \tilde{P}_2 \right| \right|_F D_{\alpha, P_1} \left(G_1^{(1)},G_1^{(2)} \right) \\
& \qquad + \left| \left| \tilde{P}_1 \right| \right|_F D_{\alpha, P_2} \left( G_2^{(1)},G_2^{(2)} \right), %
\end{align*}%
which is a weighted sum of objectives of the optimizations for prolongation from $G_1^{(1)}$ to $G_1^{(2)}$ and $G_2^{(1)}$ to $G_2^{(2)}$. Recall that our original constraint on $P$ was that $P^T P = I$; since $P = P_1 \otimes P_2$ this is equivalent (by a property of the Kronecker product; see Corollary 13.8 in \cite{laub2005matrix})  to the coupled constraints on $P_1$ and $P_2$:
\begin{align}
\label{eqn:p1p2const}
\left( {P_1}^T P_1 = \frac{1}{\eta}  I_1^{(1)} \right) \qquad \wedge \qquad \left( {P_2}^T P_2 = \eta I_2^{(1)} \right)
\end{align}
for some $\eta \in \mathbb{R}$. For any $P_1, P_2$ which obey \ref{eqn:p1p2const}, we may rescale them by $\eta$ to make them orthogonal without changing the value of the objective, so we take $\eta = 1$ in subsequent calculations.  Noting that $||A||_F = \sqrt{\text{Tr}(A^T A)}$, we see that 
\begin{align*}
\left| \left| P_1 \right| \right|_F = \sqrt{\text{Tr}( I_1^{(1)})} = \sqrt{n_1^{(1)}} \quad \text{and similarly} \quad \left| \left| P_2 \right| \right|_F  = \sqrt{ n_2^{(1)}}.
\end{align*}

Thus, we have proven the following:
\begin{theorem}
\label{thm:decompthm}
Assuming that $P$ decomposes as $P = P_1 \otimes P_2$, the diffusion distance $D_{P, \alpha} \left(G_\Box^{(1)}, G_\Box^{(2)} \right)$ between $G_\Box^{(1)}$ and $G_\Box^{(2)}$ is bounded above by the strictly monotonically increasing function of the two distances $D_{P_1,\alpha}$ and  $D_{P_2,\alpha}$:
\begin{align*} \mathcal{F}(D_{P_1,\alpha}, D_{P_2,\alpha}) &= \sqrt{n_2^{(1)}} D_{P_1,\alpha}  + \sqrt{n_1^{(1)}} D_{P_2,\alpha} ,
\intertext{Namely,}
D_{P, \alpha} \left(G_\Box^{(1)}, G_\Box^{(2)} \right) &\leq \mathcal{F} \left(  D_{P_1,\alpha} \left( G_1^{(1)},  G_1^{(2)} \right),  D_{P_2,\alpha} \left( G_2^{(1)},  G_2^{(2)} \right) \right)
\end{align*}
\end{theorem}
Thus, the original optimization over the product graphs decouples into separate optimizations over the two sets of factors, constrained to have the same value of $\alpha$. Additionally, since the requirement that $P = P_1 \otimes P_2$ is an additional constraint,
\begin{corollary}
\label{thm:decompcorollary}
If $(\alpha_1, P_1)$ and $(\alpha_2, P_2)$, subject to orthogonality constraints, are optima of $D_{\alpha, P} \left( G_1^{(1)}, G_1^{(1)} \right)$ and $D_{\alpha, P} \left( G_2^{(1)}, G_2^{(1)} \right)$, and furthermore if $\alpha_1 = \alpha_2$, then the value of $D_{P, \alpha} (G_1^{(1)} \Box G_2^{(1)}, G_1^{(2)} \Box G_2^{(2)})$ for an optimal $P$ is bounded above by $D_{P_1 \otimes P_2, \alpha_1} (G_1^{(1)} \Box G_2^{(1)}, G_1^{(2)} \Box G_2^{(2)})$.
\end{corollary}
This upper bound on the original objective function is a monotonically increasing function of the objectives for the two smaller problems. 
A consequence of this upper bound is that if ${D_{P_1,\alpha} \left( G_1^{(1)},  G_1^{(2)} \right) \leq \epsilon_1}$ and ${D_{P_2,\alpha} \left( G_2^{(1)},  G_2^{(2)} \right) \leq \epsilon_2}$, then the composite solution $P_1 \otimes P_2$ must have ${D_{P_1 \otimes P_2,\alpha} \left( G_\Box^{(1)},  G_\Box^{(2)} \right) \leq \epsilon = \left( \sqrt{n_1} + \sqrt{n_2} \right) \max(\epsilon_1, \epsilon_2)}$. Thus if both of these distances are arbitrarily small then the composite distance must also be small. Furthermore, if only one of these is small, so that ${D_{P_1,\alpha} \left( G_1^{(1)},  G_1^{(2)} \right) \approx 0}$ or ${D_{P_2,\alpha} \left( G_2^{(1)},  G_2^{(2)} \right) \approx 0}$, then ${D_{P_1 \otimes P_2,\alpha}  \approx D_{P_2, \alpha}}$ or ${D_{P_1 \otimes P_2,\alpha}  \approx D_{P_1, \alpha}}$, respectively.

We have experimentally found that many families of graphs do not require scaling between the two diffusion processes: the optimal $(\alpha, P)$ pair has $\alpha = 1$. In particular, prolongation between path (cycle) graphs of size $n$ and size $2n$ always have $\alpha_\text{optimal} = 1$, since the spectrum of the former graph is a subset of that of the larger - therefore, there is a matching solution of cost 0 which by the argument above can be mapped to a graph-space $P$ with objective function value 0 (we prove this in Section \ref{sec:zeroerrorcycle} of the Supplementary Material to this paper). In this case, the two terms of the upper bound are totally decoupled and may each be optimized separately (whereas in the form given above, they both depend on a $\alpha$).

\section{Conclusion and Future Work}
\label{sec:future}
We have introduced a novel method for multiscale modeling, which relies on a novel prolongation and restriction operator to move between models in a hierarchy. These prolongation and restriction operators are the optima of an objective function we introduce which is a natural distance metric on graphs and graph lineages. We prove several important properties of this objective function, including an upper bound which allows us to decouple a difficult optimization into two smaller optimization problems under certain circumstances. 

Additionally, we demonstrate an algorithm which makes use of such $P$ and $R$ operators to simultaneously train models in a hierarchy of neural networks (specifically, autoencoder neural networks). This Multiscale Artificial Neural Network (MsANN) approach statistically outperforms training only at the finest scale, achieving lower error than the default model and also reaching the default model's best performance in an order of magnitude fewer training examples. While in our experiments we saw uniform improvement as the parameters $\gamma$, $k$, and $L$ were increased (meaning that the hierarchy is deeper, and the model spends more relative time training at the coarser scales), this may not always be the case, and we leave the question of finding optimal settings of these parameters for future work. 

Future work will also focus on investigating the properties of the distance metric on graphs, and the use of those properties in graph lineage, as well as modifying the MsANN algorithm to perform the same type of hierarchical learning on more complicated ANN models, such as Convolutional Neural Networks (CNNs), as well as non-autoencoding tasks, for example classification.

\section*{Acknowledgments}
\ackinfo

\newpage
\bibliographystyle{siamplain}
{\tiny
\bibliography{references}}
\end{document}


\maketitle

\section{Existence of Zero-Error $P$ for Cycle Graphs}
\label{sec:zeroerrorcycle}
\begin{theorem}
\label{theorem:spectralzero}
Let $G^{(1)}$ and $G^{(2)}$ be graphs with spectra $\lambda^{(1)}_i$ and $\lambda^{(2)}_j$, respectively, with $n_1 = \left| G^{(1)} \right| \leq n_2 = \left| G^{(2)} \right|$. Suppose that for every $\lambda^{(1)}_i = r$ of multiplicity $k$, $r$ is also an eigenvalue of $G^{(2)}$ of multiplicity $\geq k$. Then there is a zero-cost eigenvalue matching $M$ between $G^{(1)}$ and $G^{(2)}$.
\end{theorem}
\begin{proof}
Let $(i_1, j_1), (i_1, j_1) \ldots (i_{n_1}, j_{n_1})$ be a list of pairs of indices such that the following hold:
\begin{itemize}
\item All of the $i_k$ are unique. 
\item All of the $j_k$ are unique.
\item For any pair $(i_k, j_k)$, $\lambda^{(1)}_k = \lambda^{(2)}_k$.
\end{itemize}
Define $P$ as follows:
\[
P_{ij} = 	\begin{cases} 
      			1 & (i,j) \text{appears in the above list.} \\
      			0 & \text{else.} \\
   			\end{cases}
\]
$P$ is clearly orthogonal, since it has exactly one $1$ in each row and each column and zeros elsewhere ($P$ is a permutation matrix for $n_1 = n_2$ and a \emph{subpermutation} matrix otherwise). Furthermore, we must have 
\[\sum_{i=1}^{n_2} \sum_{j=1}^{n_1} P_{ij} \left( \lambda^{(1)}_j - \lambda^{(2)}_i \right) = 0\]
and therefore
\begin{align}
\label{eqn:p_zero_cost}
\left|  \left| P \Lambda^{(1)} - \Lambda^{(2)} P \right| \right|_F = 0 \end{align}
\end{proof}

\begin{corollary}
For any $n$, there exist zero-error matchings between $C_n$ and $C_{2n}$.
\end{corollary}
\begin{proof}
The spectra of $C_n$ are given by the formula (see \cite{hogben2006handbook} Section 39.3):
\[ \lambda(C_n) = 2 \cos (\frac{2 \pi j}{n}) \qquad \text{for} \qquad (j=0, 1, \ldots n-1) \]
Thus $\lambda(C_n)$ and $\lambda(C_{2n})$ clearly satisfy the conditions of Theorem \ref{theorem:spectralzero} above. In particular, the matrix 
\[
P_{ij} = 	\begin{cases} 
      			1 & \text{if} \quad i = 2j \\
      			0 & \text{else.} \\
   			\end{cases}
\] has 0 cost as in Equation \ref{eqn:p_zero_cost}
\end{proof}

\section{Spectral Version of Decoupling for the Diffusion Term of Graph Product Prolongations}
\begin{theorem}
The eigenspace version of the diffusion term of the objective function of a graph product prolongation also decouples into two smaller prolongation problems.
\end{theorem}
\begin{proof}
From Theorem \ref{thm:decompthm} of the main manuscript, we know that for 
\begin{align*}
\mathbf{G}_\Box^{(1)} = G^{(1)}_1 \Box G^{(2)}_1 \qquad & \text{and} \qquad \mathbf{G}_\Box^{(2)} = G^{(1)}_2 \Box G^{(2)}_2 ,
\end{align*}
and assuming $P = P_1 \otimes P_2$,
\begin{align*}
D^{P,\alpha} \left( G_\Box^{(1)},  G_\Box^{(2)} \right) &= \left| \left| P L_\Box^{(1)} - L_\Box^{(2)} P \right| \right|_F\\
&\leq \sqrt{n_2^{(1)}} D_{P_1,\alpha} \left( G_1^{(1)},  G_1^{(2)} \right) + \sqrt{n_1^{(1)}} D_{P_2,\alpha} \left( G_2^{(1)},  G_2^{(2)} \right) \\
&= \sqrt{n_2^{(1)}} \left| \left| \frac{1}{\sqrt{\alpha}} P_1 L_1^{(1)}  - \sqrt{\alpha} L_1^{(2)}  P_1 \right| \right|_F \\
& \qquad + 
\sqrt{n_1^{(1)}} \left| \left| \frac{1}{\sqrt{\alpha}} P_2  L_2^{(1)} - \sqrt{\alpha} L_2^{(2)} P_2 \right| \right|_F , \\
\end{align*}
Trivially, we can rewrite each of these Frobenius norms to be their spectral version, as in Equation \ref{eqn:distancedefn}. Thus,
\begin{align*}
\left| \left| P L_\Box^{(1)} - L_\Box^{(2)} P \right| \right|_F &= \left| \left| \tilde{P} \Lambda_\Box^{(1)} - \Lambda_\Box^{(2)} \tilde{P} \right| \right|_F \\
&\leq \sqrt{n_2^{(1)}} \left| \left| \frac{1}{\sqrt{\alpha}} \tilde{P}_1 \Lambda_1^{(1)}  - \sqrt{\alpha} \Lambda_1^{(2)}  \tilde{P}_1 \right| \right|_F \\
& \qquad + 
\sqrt{n_1^{(1)}} \left| \left| \frac{1}{\sqrt{\alpha}} \tilde{P}_2  \Lambda_2^{(1)} - \sqrt{\alpha} \Lambda_2^{(2)} \tilde{P}_2 \right| \right|_F , \\
\end{align*}
which is a weighted sum of objectives of the two spectral prolongation problems for the two factor lineages. We have thus also decoupled this eigenvalue-matching version of the objective function into two separate prolongation problems.
\end{proof}
Finally, we show that if $\tilde{P}_1$ and $\tilde{P}_2$ are solutions to the eigenvalue matching problem $m*(L_1^{(1)}, L_1^{(2)})$ and $m*(L_2^{(1)}, L_2^{(2)})$ respectively, then $\tilde{P} = \tilde{P}_1 \otimes \tilde{P}_2$ is a valid, but not necessarily optimal, solution to the eigenvalue matching problem $m*(L_\Box^{(1)}, L_\Box^{(2)})$. By valid we mean that $\tilde{P}$ satisfies the constraints given in the definition of matching problems in Section \ref{subsub:definitions}. 
\begin{proof}
This fact follows directly from the constraints on $\tilde{P}_1$ and $\tilde{P}_2$. A matrix $M$ is a valid matching matrix iff its entries are in $\{ 0, 1\}$ and it is orthogonal (this is an equivalent expression of the constraints given in Section \ref{subsub:definitions}. If $\tilde{P}_1$ and $\tilde{P}_2$ are both orthogonal and $\{0,1\}$-valued, then we observe the following facts about their Kronecker product $\tilde{P}_1 \otimes \tilde{P}_2$:
\begin{itemize}
\item it is also $\{0,1\}$-valued, since any of its entries is the product of an entry of $\tilde{P}_1$ and one of  $\tilde{P}_2$.
\item it is orthogonal, since $A \otimes B$ is orthogonal iff $A^T A = \zeta I$ and $B^T B = \frac{1}{\zeta} I$ for some $\zeta > 0$ \cite{laub2005matrix}. $\tilde{P}_1$ and $\tilde{P}_2$ satisfy these conditions with $\zeta = 1$.
\end{itemize}
\begin{sloppypar}
\noindent So $\tilde{P}_1 \otimes \tilde{P}_2$ satisfies the constraints given for the eigenvalue matching problem ${m*(L_\Box^{(1)}, L_\Box^{(2)})}$.
\end{sloppypar}
\end{proof}

\section{Comparison of Numerical Methods}
\label{sec:sup:comparison}

To find minima of this objective function, we explore several numerical methods. For prototyping, we initially used Nelder-Mead \cite{nelder1965simplex} optimization with explicit orthogonality constraints, as implemented in the Mathematica commercial computer algebra program. However, this approach does not scale - in our hands Mathematica was not able to minimize this objective function with more than approximately $200$ unknowns in a reasonable amount of time. Our next approach was to use a special-purpose code \cite{wen2013feasible} for orthogonally-constrained gradient descent. While this software package scaled well to pairs of large graphs, it required many random restarts to find minima of our objective function. Motivated by its automatic differentiation capability and its ability to handle larger numbers of unknowns, we tried the TensorFlow minimization package \cite{abadi2016tensorflow}: first custom-written code and then a package called PyManOpt \cite{townsend2016pymanopt} which performs manifold-constrained optimization of arbitrary objective functions expressed as TensorFlow computation graphs. PyManOpt is able to perform first- and second-order minimization while staying within the constraint manifold (rather than our custom code, which takes gradient descent steps and then projects back to the constraint surface).  These latter two approaches performed best in terms of optimization solution quality, and we compare them more throughly below.

To compare the performance of the TensorFlow method and the PyManOpt method, we explore the performance of both minimization methods as the relative weight $s$ of the locality and diffusion terms is adjusted. Figure \ref{fig:paretofig} shows the tradeoff plot of the optimized unweighted value of each term as the weight parameter $s$ is tuned. The four subplots correspond to four runs of this experiment with differing sizes of graphs; in each we find optimal prolongations from a cycle graph of size $n$ to one of size $2n$. The PyManOpt-based minimization code is clearly superior, as we see a clear linear tradeoff between objective function terms as a function of $s$. The TensorFlow code which maintains orthogonality by projecting back to the Stiefel manifold falls short of this boundary in all cases. Therefore unless otherwise specified, for the rest of this paper when we discuss solving for $P$ matrices, we are reporting results of using the PyManOpt method. 

\begin{figure}
\label{fig:paretofig}
\begin{center}
\includegraphics[width=.85\linewidth]{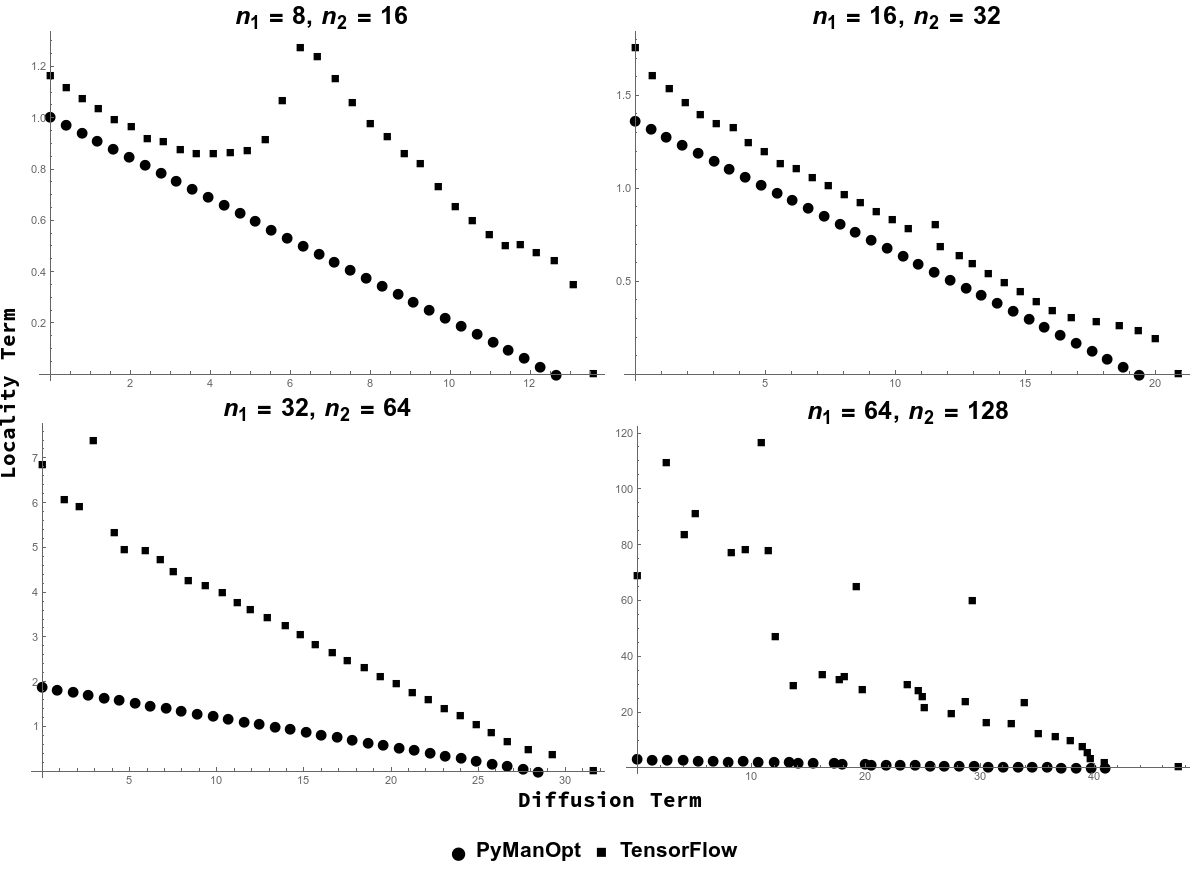}
\end{center}
\caption{Tradeoff plot of locality vs. diffusion for several pairs of graphs. Multiple solutions are plotted in each subplot, representing the adjustment of the $s$ parameter in our objective function from totally local to totally diffuse. We see that the PyManOpt boundary shows a linear tradeoff between the two terms of the objective function as their relative weight is tuned, whereas the Tensorflow boundary is more irregular. Furthermore, the PyManOpt method in general finds optima with lower objective function value than Tensorflow (for both objectives). We note that Nelder-Mead in Mathematica would not be able to tackle problems of this size, and the method due to Wen and Yin \cite{wen2013feasible} produced points which are off of this plot by at least an order of magnitude (we do not present these points).} 
\end{figure}

\section{Code}

An implementation (in TensorFlow) of the MsANN training procedure is available at 
\url{https://www.ics.uci.edu/~scottcb/research/MsANN/MsANN.html}, or via GitHub at \url{https://github.com/scottcb/MsANN}. Contact C.B. Scott with any questions about this implementation. 

\newpage
\bibliographystyle{siamplain}
\bibliography{references}